\documentclass{siamart220329}
\usepackage{verbatim}
\usepackage[utf8]{inputenc}
\usepackage[T1]{fontenc}
\usepackage{enumerate}
\usepackage{comment}
\usepackage{graphicx}
\usepackage{amsmath}
\usepackage{amsfonts}
\usepackage{algorithm}
\usepackage{algpseudocode}
\usepackage{amssymb}
\usepackage{caption}
\usepackage{outlines}
\usepackage{subcaption}
\usepackage{spacingtricks}

\newtheorem{assumption}{Assumption}
\newtheorem{configuration}{Step-size Schedule}
\newtheorem{example}{Example}
\newtheorem{remark}{Remark}

\newcommand*\tstrut[1]{\vstrut{#1}}

\title{Optimal Algorithms for Stochastic Complementary Composite Minimization}
\author{Alexandre d'Aspremont, Crist\'obal Guzm\'an, Clément Lezane}
\date{\today}

\renewenvironment{equation*}{\[}{\]\ignorespacesafterend}

\begin{document}

\maketitle

\begin{abstract}
Inspired by regularization techniques in statistics and machine learning, we study complementary composite minimization in the stochastic setting. This problem corresponds to the minimization of the sum of a (weakly) smooth function endowed with a stochastic first-order oracle, and a structured uniformly convex (possibly nonsmooth and non-Lipschitz) regularization term. Despite intensive work on closely related settings, prior to our work no complexity bounds for this problem were known. We close this gap by providing novel excess risk bounds, both in expectation and with high probability. Our algorithms are nearly optimal, which we prove via novel lower complexity bounds for this class of problems. We conclude by providing numerical results comparing our methods to the state of the art.
\end{abstract}


\section{Introduction}

Regularization is one of the most common and successful techniques in stochastic optimization. A regularized objective is given by
\begin{equation} \label{eqn:comp_min}
\begin{aligned}
\min_{x \in \mathcal{X}} \Psi(x) :=[ F(x) +  H(x) ].
\end{aligned}
\end{equation}
Here, ${\cal X}\subseteq \mathbb{R}^d$ is a closed convex set,  $F(x)=\mathbb{E}_z[f(x,z)]$\footnote{In our model, we can consider $F$ endowed with a first-order stochastic oracle, which is strictly more general than a population loss function. The latter representation is only used as a motivation.} represents an expected population loss function, and $H(x)$ is a regularization term that promotes a desired structure for the obtained solution, such as sparsity or having low norm. 

To illustrate more concretely this problem, consider a {\em generalized ridge regression} model studied in \cite{Frank:1993}. This model arises in (random design) linear regression, when applying the maximum likelihood principle under Gaussian output noise and prior parameter distribution given by a density $\propto \prod_j\exp(- |x_j|^q)$, where $1\leq q< \infty$. This family of densities models the geometry of the target predictor. The resulting model is then
\begin{equation} \label{eqn:bridge_reg}
\min_{x\in\mathbb{R}^d} \mathbb{E}_{(a,b)}[(a^{\top} x-b)^2]+\mu\|x\|_q^q.
\end{equation} 
We note this model also arises in sparse risk minimization  \cite{Koltchinskii:2009}, where $q\approx 1$.

Typically, the two functions in \eqref{eqn:comp_min} satisfy complementary properties, such as smoothness for $F$ and strong convexity for $H$. 
Further, in cases such as \eqref{eqn:bridge_reg}, $H(x)=\mu\|x\|_q^q$ is only uniformly convex (when $q\geq 2$) \cite{Ball:1994}.
In this work, we are particularly interested in situations where the underlying norm of the space is non-Euclidean: notice that this norm quantifies the smoothness and strong convexity parameters. Here, it is known that the composite objective \eqref{eqn:comp_min} may not simultaneously enjoy smoothness and strong convexity, or that its condition number may increase polynomially with the dimension\footnote{This limitation is not specific to composite objectives, but to arbitrary functions.} \cite{AGJ18,DG21}. This limitation calls for a more nuanced exploitation of the objective's structure.

The {\em complementary composite minimization} model has been recently proposed to address this limitation \cite{DG21}. Here, deterministic algorithms that combine gradient computations of $F$ with regularized proximal steps on $H$ have been proposed. Interestingly, these algorithms attain accelerated linear convergence rates with an {\em effective condition number} parameter, that is the ratio between the smoothness constant of $F$ with the strong convexity constant of $H$. Our goal in this work is to investigate algorithms for the model \eqref{eqn:comp_min} when $F$ is endowed with a {\em stochastic first-order oracle.}

\subsection{Contributions} Our work initiates the study of complementary composite minimization \eqref{eqn:comp_min} in the stochastic setting. We provide novel algorithms, matching lower complexity bounds, and conclude with numerical experiments to show the benefits of our approach, compared to the state of the art. We remark that our methods are very general, encompassing problems in the form \eqref{eqn:comp_min} where $F$ is convex and weakly smooth, and $H$ is uniformly convex.

\begin{table}[h!]
\centering
\begin{tabular}{|c|c|c|c|} 
 \hline
 Complexity & Initialization cost & Deterministic cost & Stochastic cost \\ [0.5ex] 
 \hline\hline
 NACSMD  & $\frac{L}{\mu} \log\left(\frac{V_{0}}{\epsilon}\right)$ & $\frac{L}{\mu}\left(\frac{L}{\epsilon}\right)^{\frac{q-\kappa}{\kappa}}$ & $\frac{\sigma}{\mu} \left(\frac{\sigma}{\epsilon}\right)^{q-1}$ \\
 ACSMD & $\left(\frac{L}{\mu}\right)^{\frac{1}{q}} \log\left(\frac{ V_{0}}{\epsilon}\right)$ & $\left(\frac{L}{\mu}\right)^{\frac{\kappa}{q\kappa+q-\kappa}} \left(\frac{L}{\epsilon}\right)^{\frac{q-\kappa}{q\kappa-q+\kappa}}$ & $\frac{\sigma}{\mu} \left(\frac{\sigma}{\epsilon}\right)^{q-1}$ \\
 Lower bound & $\sqrt{\frac{L}{\mu}}$ ${(\ast)}$  
 & $\left(\frac{L}{\mu}\right)^{\frac{\kappa}{q\kappa+q-\kappa}} \left(\frac{L}{\epsilon}\right)^{\frac{q-\kappa}{q\kappa-q+\kappa}}$ & $\frac{\sigma}{\mu} \left(\frac{\sigma}{\epsilon}\right)^{q-1}$ \\ 
 \hline
\end{tabular}
\caption{Summary of upper and lower complexity bounds in the paper (up to constant factors that may depend on $q,\kappa$). The complexity is decomposed as the sum of three different terms, where the first two of them also arise in deterministic settings \cite{DG21}. The results applies for $ 1 < \kappa \leq 2 \leq q < \infty $ with $\kappa < q$, except the lower bound $(\ast)$ which is applicable only when $q=2$.}
\label{table:results}
\end{table}

\paragraph{Upper Bounds} Our algorithms are inspired by the literature on stochastic acceleration for first-order methods \cite{Lan11, GL12, GL12II}. We first provide a non-accelerated algorithm, that we call the {\em non-accelerated composite stochastic mirror-descent} (NACSMD), which at every iteration it computes a stochastic gradient of $F$, and uses it to perform a proximal-type step involving the stochastic gradient of $F$ and the non-linearized $H$, which is furthermore localized by the use of a Bregman divergence term. 
Combining this method with a standard restarting scheme, linear convergence (up to a level determined by the noise) is obtained, as shown in the second row of Table \ref{table:results}. Despite this algorithm not being optimal, it is useful to illustrate the main algorithmic building blocks, and it is straightforward to analyze. 

As mentioned above, the non-accelerated algorithm is known to be suboptimal, even in the noiseless case, where $\sigma=0$ \cite{DG21}. Therefore, we propose an accelerated counterpart, 
that we call the {\em accelerated composite stochastic mirror-descent} (ACSMD). As usually in acceleration, the method involves a step similar to the non-accelerated method, which is further enhanced by two auxiliary sequences. One of them provides the sequence of points where the stochastic oracle for $F$ is queried, and the other provides the sequence of points whose objective value attains the accelerated rate. This type of acceleration does not suffice on its own to get linear convergence, and therefore a similar restarting scheme that the one used for the non-accelerated method provides linear convergence: both results can be found in Table \ref{table:results}. Interestingly, this complexity bound improves upon previous upper bounds proved in the deterministic setting (i.e., where $\sigma=0$) \cite{DG21}, showing a more subtle decomposition of the complexity into three terms: a  linearly convergent term (called initialization cost in the table) -- involving the square root of the effective condition number --, a polynomial convergence term -- which in the smooth and strongly convex case vanishes --, and a stochastic term -- involving a signal-to-noise type ratio. Further, our results in the stochastic setting are the first of their kind.

Finally, we remark that our results do not only hold in expectation, but also with high-probability. We achieve this by using concentration inequalities for martingale difference sequences  \cite{Wainwright:2019}. We establish these results 
under moment generating function (mgf) assumptions for the stochastic oracle, where these bounds are adjusted to the uniform convexity of the regularizer. This framework provides a higher flexibility and it is better suited for the noise assumptions used for the in-expectation results. Furthermore, our restarting analysis is done by studying the random deviations of {\em the whole algorithm}, without splitting the concentration analysis among rounds. This is in stark contrast of other restarting algorithms that require explicit shrinking of the optimization domain (e.g., \cite{GL12II}), which is computationally challenging and degrades the probabilistic guarantee proportionally to the number of rounds. Our advances come from the simple observation that one can unravel the complete recursion of the restarted algorithm in a path-wise way, and then establishing concentration in the usual way (with modified weights, due to the restarts).

\paragraph{Lower Bounds} Our accelerated algorithms are nearly optimal in a natural oracle model, where stochastic first-order oracle access to $F$ and full access to $H$ is assumed. This oracle is a stochastic analog of the oracle model introduced in \cite{DG21}. We extend the results of \cite{DG21}, by incorporating the impact of the stochastic oracle into the complexity. Our lower bounds combine those of the deterministic setting \cite{DG21} with an information-theoretic lower bound for stochastic noise, which is based on a Bernoulli oracle. 
This type of argument has been used in stochastic convex optimization in past work \cite{Nesterov83}, and we adapt it to incorporate the uniform convexity of the regularization term.

\paragraph{Numerical Results} We run our restarted (NAC-SMD and AC-SMD) algorithms on generalized ridge regression problems as described in eqn.~\eqref{eqn:bridge_reg}. We have tested our algorithms against the state of the art \cite{GL12,GL12II} on synthetic examples with varying dimension and smoothness parameter. These results do not only confirm the validity of our theoretical advances, but show quantitative improvements upon the state of the art, and some further practical benefits, particularly a more moderate computational overhead when the smoothness parameter is overestimated. We consider this feature important, as estimating this parameter can be difficult in practical scenarios.

\subsection{Related Work}

Stochastic convex optimization is an intensively studied topic, which has been widely used to solve large-scale machine learning problems (see e.g., \cite{Nesterov83,NJLS09,Lan11,GL12,GL12II,Sra:2011,Lan:2020}). Furthermore, the concept of regularization \cite{Tikhonov:1943}, coming from inverse problems and statistics \cite{SSBD:2014,Mohri:2018}, is a well-established and successful model for solving  ill-posed problems with theoretical guarantees. Beyond the classical theory, we emphasize that the use of regularizers that are only uniformly convex (as opposed to strongly convex) has become the focus of various works \cite{Koltchinskii:2009,Combettes:2018,Bubeck:2018,Adil:2019,Adil:2019NeurIPS}. The necessity of this assumption is crucially related to the structure of the Banach spaces where these variational problems are naturally posed.


Previous works on stochastic composite minimization (e.g., \cite{Lan11}) require strong convexity and smoothness of $F$ to attain linear convergence. For the complementary setting, where the strong convexity assumption only holds for the regularizer $H$, results are rare and typically provide upper bounds only in Euclidean settings (see, e.g.~\cite{HKP09}). Furthermore, the approach in \cite{HKP09} is not compatible with the restarting scheme  
algorithm suggested in \cite{GL12II} (called multistage in that paper); note that all existing linearly convergent methods in the stochastic setting use such restarts. 
And even for the optimal performance, the convergence proof presented in the article \cite{GL12II} requires an assumption about the proximal function to be lower bounded and upper bounded by $\|\cdot \|^2$: by contrast, our approach does not need this assumption. 

Although not particularly focused on complementary settings, the work of Juditsky and Nesterov \cite{JuditskyNes:2014} (together with the classical monograph \cite{Nesterov83}) is one of the few that studies uniformly convex objectives in the stochastic setting. The setting of this paper is slightly different from ours: the stochastic objectives considered are nonsmooth and uniformly convex, and the space is endowed with a {\em strongly convex distance generating function}. Although we can adapt our techniques to extend those of \cite{JuditskyNes:2014}, we have omitted these results for brevity. 
Other recent works  focused on weak moment assumptions for the stochastic gradients (possibly with infinite variance) \cite{Vural22}, but the approach was done only for non-smooth optimization and in a non composite setting. In particular, no form of acceleration can be obtained there.

The closest work to ours is that of deterministic complementary composite minimization \cite{DG21}, which establishes the convergence of accelerated dual averaging algorithms in this setting. This work is the main inspiration for both algorithmic design, step-size schedules, as well as the lower bounds. 
We note however that, even in the deterministic settings, our upper bounds are sharper, which we attribute to the great flexibility of our step-size policy and restarting schedule. 
Independently, in \cite{Cohen:2020} composite acceleration under the lens of relative Lipschitzness and relative strong convexity was obtained by application of extragradient-type algorithms. Again, this derivation \cite[Thm.~4]{Cohen:2020} is only made for deterministic objectives.

At the technical level, we have extended the proof in \cite{GL12} to exploit the uniform/strong convexity of the regularizer, and our analysis provides a more flexible choice of step-size parameters. For the accelerated method, we also mix the AGD+ step-size from \cite{DG21}, with the usual ones in \cite{GL12} to create our own sequence of step-sizes, a particular point is that the choice becomes more intuitive and it is not unique anymore. 


Independently and concurrently to our work, Dubios-Taine at al.~\cite{DuboisTaine:2022} studied stochastic composite minimization in a smooth plus strongly convex setting (a particular case of our work). However, their algorithm only obtains constant accuracy under constant noise, as opposed to our vanishing and optimal accuracy bounds. Here as well, it appears that our advantage comes from the flexibility of the step-size schedule.

\section{Preliminaries}



We introduce here several notions which are relevant for our work. In what follows, we let $\overline{\mathbb{R}}:=\mathbb{R}\cup\{+\infty\}$. For the algebraic and ordering properties of this space, see e.g., \cite{Beck:2017}.


\begin{definition}[Uniform convexity]
Let $q\geq 2$ and $\mu>0$. A function $f:\mathcal{X} \xrightarrow{} \overline{\mathbb{R}}$ subdifferentiable on its domain is $(\mu,q)$-uniformly-convex w.r.t. a norm $\|\cdot\|$ if 
\[ f(x) - f(y) - \langle g, x-y \rangle \geq \frac{\mu}{q} \|x-y\|^q \quad(\forall x\in {\cal X}, \forall y\in\mbox{\em dom}(f), \exists g\in\partial f(y)).\]
\end{definition}

\begin{example}
In the case of $q=2$, the definition of uniform convexity coincides with the more well-known notion of strong convexity. In that case, it is known that the function $f(x)=\frac{1}{2}\|x\|_p^2$, where $1<p\leq 2$, is $(p-1,2)$-uniformly convex w.r.t.~ $\|\cdot\|_p$. 
Another example, the negative entropy, defined as $f(x)=\sum_{j=1}^d x_j\ln(x_j)$ if $x\in \Delta_d$ (the standard unit simplex in $\mathbb{R}^d$), and $+\infty$ otherwise; is $(1,2)$-uniformly convex w.r.t.~$\|\cdot\|_1$. For these two examples we refer the reader to \cite[Section 5.3.2]{Beck:2017}.

Now let us consider the case of $q>2$. Then, it is possible to show that $f(x)=\|x\|_q^q$ is $(2^{-\frac{q(q-2)}{q-1}}/q,q)$-uniformly convex w.r.t.~$\|\cdot\|_q$. We provide more details in Appendix \ref{uniform convexity example}. All the previous examples can be extended to their spectral counterparts, namely the Schatten spaces $\mbox{Sch}_{p}:= (\mathbb{R}^{d\times d},\|\cdot\|_{Sch,p})$, where given the spectrum of a matrix $X$, $(\sigma_j(X))_{j\in[d]}$, we define its Schatten norm as $\|\cdot\|_{Sch,p}:=\sum_{j}|\sigma_j(X)|^q$(see e.g.~\cite{Ball:1994,Juditsky:2008}). These matrix counterparts arise naturally in linear inverse problems \cite{Recht:2010, NesterovNemirovski:2013}.
\end{example}

In what follows, we   denote by $\nabla f(y)$ any subgradient of a function $f$ at point $y$. This is only for notational convenience, and it can be done without loss of generality.

\begin{definition}[Weak smoothness]
Let $\kappa\in(1,2]$ and $L\geq 0$. A differentiable function $f:\mathcal{X} \xrightarrow{} \mathbb{R}$ is $(L,\kappa)$-weakly smooth w.r.t. a norm $\|\cdot\|$ if 
\[ f(x) - f(y) - \langle \nabla f(y), x-y \rangle \leq \frac{L}{\kappa} \|x-y\|^\kappa \quad(\forall x,y\in {\cal X}).\]
\end{definition}

\begin{definition}[Bregman Divergence]
Let $\omega:\mathcal{X} \xrightarrow{} \mathbb{R}$ be a convex function which is continuously-differentiable on the interior of its domain, we define the Bregman divergence of $\omega$ as 
\[ D^{\omega}(x,y) = \omega(x) - \omega(y) - \langle \nabla \omega(y), x-y \rangle \qquad(\forall x\in {\cal X}, y\in \mbox{\em dom}(\omega)).\]
Note that if $\omega$ is (uniformly) convex, then the Bregman divergence is (uniformly) convex on its first argument.
\end{definition}




The following result is a consequence of the three-points identity \cite{CT93}. We note that \cite{Lan11} follows a similar route, where a negative Bregman divergence term is upper bounded by zero. We maintain this term, as it is crucial for our improved rates.

\begin{lemma}  \label{lem:prox}
Let $f$ be a convex function and $\nu$ be convex and continuously differentiable. If we consider \[ u^{\star} = \arg\min_{u\in {\cal X}} \{f(u) + D^{\nu}(u,y)\}, \]
then for all $u$:
\[f(u^{\star}) + D^{\nu}(u^\star,y) + D^{f}(u,u^\star) \leq f(u) +  D^{\nu}(u,y) - D^{\nu}(u,u^\star). \]
\end{lemma}

\begin{proof}

From the first order optimality conditions, for all $u\in{\cal X}$: 

\[ \langle \nabla f(u^\star) + \nabla D^{\nu}(u^{\star},y) , u - u^{\star}\rangle \geq 0,\] with the gradient taken with respect to the first entry. We also apply the three-points identity from \cite{CT93}: 

\[\langle \nabla D^{\nu}(u^{\star},y) , u - u^{\star}\rangle = D^{\nu}(u,y)- D^{\nu}(u,u^{\star})- D^{\nu}(u^{\star},y).\]

Then:
\begin{align*}
  f(u) - f(u^{\star}) - D^{f}(u,u^{\star}) & = \langle \nabla f(u^\star) , u - u^{\star}\rangle \\
    & \geq D^{\nu}(u,u^{\star})- D^{\nu}(u,y)  +D^{\nu}(u^{\star},y).
\end{align*}

\end{proof}

Given parameters $L,\kappa, \mu,q$ we define the following  parameters to simplify notation: 
\begin{equation} \label{eqn:notations_param}
    \begin{aligned}
    \begin{cases}
    r & := \frac{q-\kappa}{\kappa} \geq0 \\
    M & :=   \big( \frac{r}{q} \big)^{r} L \geq 0\\
    p & := \frac{q}{q-1} \in(1,2].
    \end{cases}
    \end{aligned}
\end{equation}

Now we introduce a key lemma, which first arose in a more restricted form in \cite{DGN13} in the context of methods with inexact gradients, and has later been used to bridge uniform convexity and uniform smoothness inequalities in first-order methods as in \cite{Nesterov14}, \cite{AGJ18} and \cite{DG21}. Here, we use a homogeneous version of the lemma. 

\begin{lemma}
\label{Lemma inexact gradient}
From notation \ref{eqn:notations_param}, if $f$ is $(L,\kappa)$-weakly smooth, then for all $\delta>0$ 
\begin{equation}
\begin{aligned}
f(x) - f(y) - \langle \nabla f(y), x-y \rangle \leq \frac{M}{q \delta^r} \|x-y\|^q + L \delta \qquad(\forall x,y\in \mathcal{X}).
\end{aligned}
\end{equation}
\end{lemma}

\begin{proof}
For $x,y \in \mathcal{X}$, we know $f(x) \leq f(y) + \langle \nabla f(y), x-y \rangle + \frac{L}{\kappa}\|x-y\|^{\kappa}.$ 
Now use the Young inequality as in \cite{AGJ18}, for $\frac{1}{a} + \frac{1}{b} = 1 $ we have $t \leq \frac{1}{az}t^a + \frac{1}{b}z^{b-1}.$  We consider $t:= \frac{\|x-y\|^\kappa}{\kappa}$, $a=\frac{q}{\kappa}$, $b=\frac{q}{q-\kappa}$, $z= (b\delta)^{\frac{1}{b-1}}= (b\delta)^{r}$, we scale everything with $L$: 
\[ \frac{L}{\kappa}\|x-y\|^{\kappa} \leq  \frac{M}{q \delta^r} \|x-y\|^q + L \delta, \]
where the middle term comes from
\[\frac{\kappa L}{q} \left( \frac{q-\kappa}{q \delta}\right)^{r} \left(\frac{1}{\kappa}\right)^{q/\kappa} = \frac{L}{q \delta^r}\left(\frac{q-\kappa}{q \kappa}\right)^{r} = \frac{M}{q \delta^r}, \]
plugging this bound back in the first step of the proof shows the result. 
\end{proof}

\subsection{The Stochastic Oracle Model}

We are interested in studying problem \eqref{eqn:comp_min} in a natural oracle model, that we will refer to as the {\em stochastic composite oracle model}. We make the following assumptions:
\begin{itemize}
    \item $F:{\cal X}\mapsto\mathbb{R}$ is convex and $(L,\kappa)$-weakly smooth.
    \item $H:{\cal X}\mapsto\overline{\mathbb{R}}$ is $(\mu,q)$-uniformly convex, continuously differentiable, and dom$(H)\neq\emptyset$.
\end{itemize}
Notice that under these assumptions, problem \eqref{eqn:comp_min} has a unique solution, that we will denote by $x^{\star}$.

Now we proceed to specify the oracle assumptions for both functions.
\begin{assumption}\label{assump:gradient_noise}
From the notation introduced in \eqref{eqn:notations_param}, we assume the existence of an oracle that for any given $x\in {\cal X}$ provides a random variable $G(x,\xi)$ such that
\begin{eqnarray}
 \mathbb{E}_{\xi}[G(x,\xi)] &=&\nabla F(x). \label{eqn:unbiased} \\
 \mathbb{E}_{\xi}[\|G(x,\xi)-\nabla F(x)\|_{\ast}^p] &\leq& \sigma^p. \label{eqn:moment} 
\end{eqnarray}
\end{assumption} 
The first equation states that that $G(x,\xi)$ is an unbiased estimator of the gradient $\nabla F(x)$. On the other hand, the second equation controls the $p$-th moment of the noise of this oracle. Notice that by the Jensen inequality this assumption is more restrictive for higher values of $p$.


Our algorithms will be based on the  mirror-descent method. For this, we will use the regularizer $H$ as our distance-generating function (dgf) which is continuously differentiable on the interior of its domain and is $(1,q)$-uniformly convex. 
We introduce the standard assumption on the computability of the prox-mapping for the dgf \cite{NJLS09}.
\begin{assumption}\label{assump:subproblem}
We assume that for any linear function $g\in\mathbb{R}^d$, the problem below can be solved efficiently,
\begin{equation} \label{eqn:prox_step}
  \min_{x\in {\cal X}} [\langle g,x\rangle+H(x)]. 
\end{equation}
\end{assumption}
Notice also that Assumption \ref{assump:subproblem} implies the computability of subproblems involving the Bregman divergence, $\min_x[\langle g,x\rangle+H(x)+D^{H}(x,y)],$ for any $g,y\in \mathbb{R}^d$.

For convenience, we introduce the {\em gradient noise} random variable, 
\[\Delta(x) := G(x, \xi) - \nabla F(x).\]
To derive high probability accuracy bounds, we will use Bernstein-type concentration inequalities \cite{Wainwright:2019}, with adaptations regarding the exponent $p$. The usual assumption in the literature relates to a sub-Gaussian tail bound on the norm of the stochastic oracle error (see, e.g. \cite{NJLS09,Lan11,GL12,GL12II}). However, weaker moment bounds such as Assumption \eqref{assump:gradient_noise} with $1<p<2$ are inconsistent with sub-Gaussian tails. 

\begin{assumption} 
\label{assumption:mgf_surrogate}
We assume that given a sample $\xi$, for all $x\in {\cal X}$ 
\begin{equation}\label{eqn:exp_mgf}
\mathbb{E}\Big[\exp\Big\{\frac{\|\Delta(x) \|_{\ast}^p}{\sigma^p} \Big\}\Big] \leq 2.
\end{equation}
\end{assumption} 

We notice that this assumption implies the bound \eqref{eqn:moment} in Assumption \ref{assump:gradient_noise}: 
\[ \mathbb{E} \Big[\frac{\| \Delta(x)\|_{\ast}^p}{\sigma^p} \Big] \leq \mathbb{E} \Big[ \exp \Big( \frac{\| \Delta(x)\|_{\ast}^p}{\sigma^p} \Big)-1 \Big] \leq 1 \quad\implies \quad\mathbb{E} \Big[\| \Delta(x)\|_{\ast}^p \Big] \leq \sigma^p. \] 
This assumption gives also an upper bound for inner products of the gradient noise, which is straightforward from the H\"older inequality, thus we omit its proof.
\begin{corollary} 
Suppose that ${\cal X}\subseteq {\cal B}_{\|\cdot\|}(x_{\star},R)$, for some $R>0$. Then, under Assumption \ref{assumption:mgf_surrogate}, if we let $W := \langle \Delta(x),x_{\star}-x \rangle$, then:
\begin{equation} \label{eqn:mgf_real} \mathbb{E}\Big[\exp\Big\{\frac{|W|^p}{\sigma^p R^p} \Big\}\Big] \leq 2. \end{equation}
\end{corollary}

In Appendix \ref{app:concentration} we derive the necessary concentration inequalities for these random variables, as well as their respective martingales. Although these results are not entirely new (see e.g., \cite{Buldygin:2000,Zajkowski:2020}), we include these analyses for completeness, and since they are not common in the optimization community. Moreover, our derivations work directly on the moment generating functions, avoiding the smoothing (also called ``standardization'') approaches carried out in the aforementioned works.

\subsection{Restarting scheme}
Finally, regarding linear convergence rates, 
there is a key technique of {\em restarting an algorithm} multiple times to reduce the initialization error exponentially fast. In our context, the idea was introduced 
in \cite{GL12II} with the restarting procedure occurring at every $2^n$-th iteration. In that reference, the authors have also used the assumption of an existing strongly convex proximal function (distance generating function), which is not needed in our algorithm.

Here we will suggest a simpler analysis of the restarting algorithm for faster convergence in expectation. Instead of restarting every $2^n$ iterations, we will simply restart the algorithm periodically. 


\begin{algorithm}[H]
\caption{Restarting Algorithm}
\label{Restarting Algorithm}
\begin{algorithmic}
\Require First stage iteration number $n,K \geq 0$, Second stage iteration number $T \geq 0$, Starting point $x_1 \in \mathcal{X}$, algorithm $\mathcal{A}$
\State Consider initial start point: $x_{1}^{0}= x_1$
\For{$ 1 \leq k \leq n$}
\State Run Algorithm $\mathcal{A}$ with $K$ iterations, obtain $(x_{K+1}^{k},y_{K+1}^{k} )$ as output. 
\State Set the new restarting point $x_{1}^{k+1} \xleftarrow{} x_{K+1}^{k}$.
\EndFor
\State Run algorithm $\mathcal{A}$ with $T$ iterations using the last  starting point $x_{1}^{n+1} = x_{K+1}^{n}$, obtain $(X_{T+1},Y_{T+1})$ and output.
\end{algorithmic}
\end{algorithm}

The general idea of the restarting scheme is to fully use the recursive form of the Bregman divergence term which appear on both sides of the accuracy guarantee (we will show that our algorithms have this feature in the next section). Leveraging that implicit distance guarantee, the algorithm can exponentially boost its convergence rate by only increasing its other polynomially convergent terms by an absolute constant factor (namely, 3). The following lemma, whose proof is deferred to Appendix \ref{annexes:proof restart}, illustrates this.

\begin{lemma}
\label{lemma restarting}
Consider $(x_{T+1},x_{T+1}^{ag})$ the output of an algorithm. If there exist $K_1,K_2,K_3,K_4\geq 0$ and  $\alpha_1 \geq \alpha_2,\alpha_3,\alpha_4$ such that we know for all $T$:
\begin{multline*}
\Psi(x_{T+1}^{ag}) -\Psi^{\star} + D^{H}(x_{\star},x_{T+1}) \leq \textstyle \frac{K_1 D^{H}(x_{\star},x_1)}{T^{\alpha_1}} 
+ \frac{K_2}{T^{\alpha_2}} + \frac{K_3}{T^{\alpha_3}} \sum_{t=1}^T Z_t + \frac{K_4}{T^{\alpha_4}} \sum_{t=1}^T W_t.
\end{multline*} 
where $\{Z_t\}_t$ and $\{W_t\}_t$ are random variables. 
Then for the output $Y_{T+1}$ of the Restarting Algorithm \ref{Restarting Algorithm} with  $K = \lceil (2K_1)^{\frac{1}{\alpha_1}} \rceil$, we have that
\begin{align*}
\Psi(Y_{T+1}) -\Psi^{\star}  & \leq \frac{D^{H}(x_{\star},x_1)}{2^{n+1}} + \frac{3K_2}{T^{\alpha_2}} + \frac{K_3}{T^{\alpha_3}} \Big( \sum_{t=1}^T  Z_t +  \underbrace{\frac{T^{\alpha_3 -\alpha_1}}{K^{\alpha_3}} \sum_{k=1}^n \sum_{t=1}^K  \frac{Z_{t}^k}{2^{n-k}}}_{\lesssim \sum_{t=1}^K  Z_{t}} \Big) \\
& + \frac{K_4}{T^{\alpha_4}} \Big( \sum_{t=1}^T  W_t +  \underbrace{\frac{T^{\alpha_4 -\alpha_1}}{K^{\alpha_4}} \sum_{k=1}^n \sum_{t=1}^K  \frac{W_{t}^k}{2^{n-k}}}_{\lesssim \sum_{t=1}^K  W_{t}} \Big).
\end{align*}
\end{lemma}
We precise here that 
$ Z_1 \lesssim Z_2$ means that we have greater upper-bounds for $Z_2$ than for $Z_1$  both in expectation and with high probability, up to constant factors. Here $D^{H}(x_{\star},x_1)$ is the implicit distance that we are reducing and all terms related to $K_2,K_3,K_4$ are the increased costs.

\section{Algorithms}

\label{sec:algorithms}

We now proceed to introduce and analyze the algorithms for the stochastic complementary composite minimization problem. For the convergence analysis, we need to introduce a partial linearization of the objective $l_{\Psi}(x,x_{t})$.:
\[l_{\Psi}(x,x_{t}) := F(x_{t}) + \langle \nabla F(x_{t}), x- x_{t}\rangle + H(x).\]
Notice that if $H=0$, this corresponds to the first order Taylor approximation of the objective, however in the complementary composite setting, we only linearize the term that can be linearly approximated, namely $F$. By convexity and weak smoothness of $F$, we know for all $\delta >0$:

\begin{equation}
    \label{proximal gradient inequality}
    \Psi(x) - \frac{M}{q \delta^r} \|x_{t}-x\|^q - L \delta \leq l_{\Psi}(x,x_{t}) \leq \Psi(x).
\end{equation}

\subsection{Non-Accelerated Method}
Our first method is a {\em non-accelerated composite stochastic mirror-descent} (NACSMD) method. This method has some resemblance to the classical stochastic mirror-descent method \cite{Nesterov83,NJLS09}, with the difference that $H$ is not linearized in the subproblem, an idea that traces back to the proximal-gradient and composite minimization literature \cite{Beck:2009,Nesterov:2013}. 

\begin{algorithm}
\caption{Non-Accelerated Composite Stochastic Mirror-Descent (NACSMD)}
\label{Algo NACSMD}
\begin{algorithmic}
\Require Number of iterations $T \geq 0$, starting point $x_1 \in \mathcal{X}$, step-sizes $(\alpha_t,\gamma_t)_t$
\For{$ 1 \leq t \leq T$} 
\begin{equation}
 x_{t+1}  = \arg\min_{x\in {\cal X}} \{ \alpha_t [ \langle G(x_t,\xi_t) , x \rangle + H(x) ] + \gamma_t D^{H}(x,x_{t}) \}
\end{equation}
\EndFor
\State Output $x_{T+1}^{ag}:=  \frac{\sum_{t=1}^{T}  \alpha_{t} x_{t+1}}{\sum_{t=1}^{T} \alpha_{t}}$. 
\end{algorithmic}
\end{algorithm}

The updates above require two sequences of step-sizes $\{\alpha_t, \gamma_t\}$. On the other hand, we require the step-size schedule to satisfy the following conditions.
\begin{configuration} \label{configuration:stepsize_linear}
Let  $\{\alpha_t, \gamma_t\}_{t \geq 1}$ be  such that for all $t\geq 1$:
\begin{equation} 
    \begin{aligned}
    \begin{cases}
    \alpha_t \geq \gamma_{t+1} - \gamma_{t} \\
     \gamma_t  \geq \frac{2M}{\mu} \alpha_t.
    \end{cases}
    \end{aligned}
\end{equation}
\end{configuration}
We notice that the constraints above have multiples solutions; for example, we can consider polynomial step-sizes with various degrees. This implies a high degree of flexibility for our methods.  The related convergence rates will be derived from the following result.

\begin{theorem} \label{thm:gen_conv_NACSMD}
Suppose Algorithm \ref{Algo NACSMD} runs under the step-size schedule \ref{configuration:stepsize_linear}. Then, if $q > \kappa$ and if we let $A_T := \sum_{t=1}^{T} \alpha_t$, for all $x \in \mathcal{X}$: 
\begin{multline} \label{eqn:NACSMD}
 A_T [ \Psi(x_{T+1}^{ag})- \Psi(x)]  +\gamma_T D^{H}(x,x_{T+1}) 
 \leq \gamma_1 D^{H}(x,x_{1})  \\ + \sum_{t=1}^{T} \alpha_t \langle \Delta_{t}(x_t) , x -x_{t}  \rangle 
+ \sum_{t=1}^{T} \frac{2\| \Delta_{t}(x_t)\|_{\ast}^{p}}{p \mu^{p/q}}  \left(\frac{\alpha_{t}^{q}}{\gamma_t}\right)^{p/q} + L \sum_{t=1}^{T} \alpha_t \left(\frac{2M\alpha_t}{\mu\gamma_t}\right)^{\frac{1}{r}},
\end{multline}
where $p,r,M$ are defined in eqn.~\eqref{eqn:notations_param}.
\end{theorem}

Notice the result above provides both a guarantee on the optimality gap and on the distance to the optimal solution (when choosing $x=x^{\star}$), as is expected for a uniformly convex program.

\begin{proof}

By the proximal lemma (Lemma \ref{lem:prox}) applied to $u=x$, $y=x_t$,  $u^{\star}=x_{t+1}$,  $f(\cdot) = \alpha_t[l_{\psi}(\cdot,x_{t}) + \langle \Delta(x_{t}),\cdot \rangle]$, and $\nu(\cdot)=\gamma_t H(\cdot)$, we have:
\begin{align*}
& \alpha_t l_{\Psi}(x_{t+1},x_{t}) + \gamma_t D^{H}( x_{t+1}, x_t) + \alpha_t D^{H}(x,x_{t+1}) \\
\leq & \alpha_t l_{\Psi}(x,x_{t}) + \gamma_t [D^{H}(x, x_{t}) - D^{H}(x,x_{t+1})] + \alpha_t \langle \Delta(x_{t}),x - x_{t+1} \rangle.
\end{align*}

Combining with the inequality of the proximal gradient equation \eqref{proximal gradient inequality} and adding $-\alpha_t \Psi(x)$ on both sides, we have for all arbitrary gap $\delta_t > 0 $ from Lemma \ref{Lemma inexact gradient}:  
\begin{multline}
\alpha_t \left[ \Psi(x_{t+1})  -\frac{M}{q\delta_{t}^r} \| x_{t+1}-x_t\|^q - L \delta_t  - \Psi(x) \right]  + \alpha_t D^{H}(x,x_{t+1})  \\
\hspace{-3.3cm}\leq \alpha_t \left[ l_{\Psi}(x_{t+1},x_{t}) -\Psi(x) \right]  + \alpha_t D^{H}(x,x_{t+1})  \\
\hspace{-1.8cm}\leq  \alpha_t \left[ l_{\Psi}(x,x_{t}) -\Psi(x) \right] + \gamma_t [D^{H}(x, x_{t}) - D^{H}(x,x_{t+1})] \\
\hspace{-3.5cm}+ \alpha_t \langle \Delta(x_{t}),x - x_{t+1} \rangle  - \gamma_t D^{H}( x_{t+1}, x_t) \\
\leq  \gamma_t [D^{H}(x, x_{t}) - D^{H}(x,x_{t+1})] + \alpha_t \langle \Delta(x_{t}),x - x_{t+1} \rangle  - \frac{\mu \gamma_t}{q} \|x_{t+1}-x_t \|^q,  
\end{multline}
where in the last inequality we used the convexity of $F$ to upper bound $l_{\Psi}(x,x_t)-\Psi(x)\leq 0$, as well as the $(\mu,q)$-uniform convexity of $H$ to upper bound the last term.

Now we need to give an upper bound $\alpha_t \langle \Delta(x_{t}),x_t - x_{t+1} \rangle$. For this, we will use the Young inequality, we will fix the gap value $\delta_t^{r} = \frac{2M \alpha_t}{\mu \gamma_t}$  :

\begin{equation*}
\begin{aligned}
\alpha_t \langle \Delta_{t}(x_t) , x_t -x_{t+1}  \rangle   & \leq \frac{\| \alpha_t \Delta_{t}(x_t)\|_{\ast}^{p}}{p  (\mu \gamma_t -\frac{M}{ \delta_{t}^{r}} \alpha_t)^{p/q}}  + \left( \frac{\mu \gamma_t}{q} - \frac{M \alpha_t}{q \delta_{t}^r}   \right) \| x_t -x_{t+1}\|^q \\
& = \frac{\| \alpha_t \Delta_{t}(x_t)\|_{\ast}^{p}}{p  (\frac{\mu \gamma_t}{2} )^{p/q}}  + \left( \frac{\mu \gamma_t}{q} - \frac{M\alpha_t}{q \delta_{t}^r}   \right) \| x_t -x_{t+1}\|^q.
\end{aligned}
\end{equation*}

Combining with the above bounds, we have: 
\begin{multline}
 \alpha_t \left[ \Psi(x_{t+1})  - \Psi(x) \right]  + \alpha_t D^{H}(x,x_{t+1}) \\
\textstyle \leq  \gamma_t [D^{H}(x, x_{t}) - D^{H}(x,x_{t+1})]  +\alpha_t \langle \Delta(x_{t}),x - x_{t} \rangle  + \frac{2 \alpha_t^p \| \Delta_{t}(x_t)\|_{\ast}^{p}}{p (\mu \gamma_t )^{p/q}} \!\!
+ L\alpha_t \left(\frac{2M\alpha_t}{\mu\gamma_t}\right)^{\frac{1}{r}}.
\end{multline}
Summing the previous equation from $t=1$ to $t=T$, we get 
\begin{multline}
\sum_{t=1}^{T} \alpha_t [ \Psi(x_{t+1}) - \Psi(x)  ] 
\leq \gamma_1 D^{H}(x,x_1) -\gamma_T D^{H}(x,x_{T+1}) \\
 \begin{array}{ll}
 &\displaystyle  + \sum_{t=1}^{T-1} (\gamma_{t+1} - \gamma_{t}) D^{H}(x,x_{t+1})  - \sum_{t=1}^{T}  \alpha_t D^{H}(x,x_{t+1})  \\
 &\displaystyle+  \sum_{t=1}^{T} \alpha_t \langle \Delta_{t}(x_t) , x -x_{t}  \rangle + \sum_{t=1}^{T} \frac{2\| \Delta_{t}(x_t)\|_{\ast}^{p}}{p \mu^{p/q}} 
 \left(\frac{\alpha_{t}^{q}}{\gamma_t}\right)^{p/q} + L \sum_{t=1}^{T} \alpha_t \left(\frac{2M\alpha_t}{\mu\gamma_t}\right)^{\frac{1}{r}}.
 \end{array}
\end{multline}

Now we will use the convexity of the problem. From the Jensen inequality, as $\Psi $ is convex,  we can aggregate the left hand side by considering the weighted sequence, $ x_{T+1}^{ag} = \frac{\sum_{t=1}^{T}  \alpha_{t} x_{t+1}}{\sum_{t=1}^{T} \alpha_{t}}.$ 
Then, after rearranging terms for all $x\in \mathcal{X}$: 

\begin{multline}
\Big(\sum_{t=1}^{T} \alpha_t \Big) [ \Psi(x_{T+1}^{ag})- \Psi(x)]  +\gamma_T D^{H}(x,x_{T+1}) \\
\leq  \gamma_1 D^{H}(x,x_{1})  + \sum_{t=1}^{T} \alpha_t \langle \Delta_{t}(x_t) , x -x_{t}  \rangle + L \sum_{t=1}^{T} \alpha_t \left(\frac{2M\alpha_t}{\mu\gamma_t}\right)^{\frac{1}{r}}\\
 + \sum_{t=1}^{T} \frac{2\| \Delta_{t}(x_t)\|_{\ast}^{p}}{p \mu^{p/q}} \left(\frac{\alpha_{t}^{q}}{\gamma_t}\right)^{p/q}  +   \underbrace{\sum_{t=1}^{T-1} [(\gamma_{t+1} - \gamma_{t} -  \alpha_t) D^{H}(x,x_{t+1})]}_{\tstrut{1.5ex} \leq 0},
\end{multline}
concluding the proof.
\end{proof}

\paragraph{Polynomial step-sizes under schedule \ref{configuration:stepsize_linear}} The previous result applies almost surely and under any step-size sequence (as long as the Step-Size Schedule  \ref{configuration:stepsize_linear} constraint is satisfied). We now focus on a family of step-sizes that increase polynomially with $t$,  
\begin{equation} \label{NACSMD_stepsize}
    \begin{aligned}
    \begin{cases}
     m & = \max \Big(\frac{1}{r} - 1, \frac{2-q}{q-1} \Big) \\
     \alpha_t & = \Big(t+2(n+1)\frac{M}{\mu}+1 \Big)^m \\
     \gamma_t & = \frac{1}{m+1} \Big(t+2(m+1)\frac{M}{\mu} \Big)^{m+1}.
    \end{cases}
    \end{aligned}
\end{equation}

The first condition comes from the fact that we need $\sum_t \delta_t^r$ and $\sum_t \left(\frac{\alpha_{t}^{q}}{\gamma_t}\right)^{p/q} $ to be divergent, that require a minimum degree of polynomials. We also notice that one can choose higher polynomial degree $m$, but that would change the convergence only up to a constant factor. Now we present the final result for NACSMD. 

\begin{theorem} \label{thm:NACSMD}
Under Assumption \ref{assump:gradient_noise} and choosing step-sizes as in \eqref{NACSMD_stepsize}, if $q > \kappa$, we have for all $T \geq 1$:
\begin{equation}
\mathbb{E} \left[\Psi(x_{T+1}^{ag}) - \Psi^{\star} + D^{H}(x_{\star},x_{T+1}) \right] \leq O_{q,\kappa} 
\left(\frac{L^{m+1} V_0 }{(\mu T)^{m+1}}  +  \left(\frac{L}{\mu}\right)^{\frac{1}{r}} \frac{L}{T^{\frac{1}{r}}} + \frac{\sigma^{p}}{(\mu T)^{p/q}} \right), \label{eqn:NACSMD_expectation_UB}
\end{equation}
where $V_0 := D^{H}(x_{\star},x_1)$ and $O_{q,\kappa}$ omits absolute constants that depend on $q,\kappa$. Moreover, if $\mathcal{X}$ is bounded with  diameter $R$ and if we note $\mathfrak{M}_{T}(\Psi)$ the previous upper bound for $\Psi$, under Assumption \ref{assumption:mgf_surrogate}, we have for all $\Omega >0$:
\begin{equation}
\begin{aligned}
&\mathbb{P}\left[\Psi(x_{T+1}^{ag}) - \Psi^{\star} \geq \mathfrak{M}_{T}(\Psi) + O_{q,\kappa} \left(\frac{\Omega \sigma R}{\sqrt{T}} + \frac{\Omega \sigma^p}{(\mu T)^{p/q}} \right)\right]\\
\leq &  \begin{cases}
\exp(-\Omega) + \exp\Big\{-\frac{1}{4}   \Omega^2 \Big\} \quad & \textit{if} \quad  \Omega \lesssim \sqrt{T} \\
\exp(-\Omega) + \exp\Big\{-\frac{1}{p}    (T^{\frac12-\frac1p}\Omega)^p \Big\} \quad & \textit{if} \quad \Omega \gtrsim \sqrt{T}. 
\end{cases}
\end{aligned}
\end{equation}
with $O_{q,\kappa}$ omits another absolute constant that depend on $q,\kappa$.
\end{theorem}

\begin{remark}
For the case where $q=\kappa=2$, we note that the inexact gradient trick is not needed to obtain the convergence rate. A similar in-expectation result can be obtained:
 \begin{equation}
\mathbb{E} \left[\Psi(x_{T+1}^{ag}) - \Psi^{\star} + D^{H}(x_{\star},x_{T+1}) \right] \leq O 
\left(\frac{L^{2} V_0 }{(\mu T)^{2}} + \frac{\sigma^{2}}{\mu T} \right).
\end{equation}
And the concentration result stays the same. The details are left to the reader.
\end{remark}

The proof for the in-expectation guarantee follows directly from Theorem \ref{thm:gen_conv_NACSMD}. For the high-probability guarantee, we defer its proof to Appendix \ref{Proof Concentration}. 

The in-expectation guarantee \eqref{eqn:NACSMD_expectation_UB} in Theorem \ref{thm:gen_conv_NACSMD} shows a decomposition of the accuracy in three terms, that we denote respectively by {\em initialization}, {\em geometric gap} and {\em variance}. Regarding the initialization, we expect that in uniformly convex settings this term can be decreased exponentially fast, which we will achieve by a restarting strategy; the geometric gap exhibits the polynomial convergence rates observed in (non-strongly) convex optimization; finally, the last term reflects the statistically optimal rates inherent to stochastic convex optimization.

\paragraph{Restarting Algorithm and Linear Convergence Rates}

Lastly, we want to apply the restarting algorithm (Algorithm \ref{Restarting Algorithm}) presented before to reduce the complexity of the first term. The restarting lemma (Lemma \ref{lemma restarting}) reduces the complexity due to initialization term with $V_0 = D^{H}(x_{\star},x_1)$. To apply the lemma, the coefficient and degree related to initialization term will be fixed at $K_1 \propto (L/\mu)^{n+1},\alpha_1 = n+1.$ Similarly for the geometric gap term  $ K_2 \propto \Big(\frac{L}{\mu}\Big)^\frac1r,\alpha_2 = \frac1r $ for the variance term $K_3 Z_t \propto \frac{\|\Delta(x_t)\|_\ast^p}{\mu^{p/q}},\alpha_3 = p/q$ and for the centered noise term $ K_4 W_t \propto \alpha_t \langle \Delta(x_t), x_{\star} - x_t\rangle, \alpha_4 = 1 $. We precise that $\propto$ means proportional up to a constant that depends on $(q,\kappa)$ only. Therefore the final complexity for the in-expectation bound is: 

\[
O_{q,\kappa}\left(\frac{L}{\mu} \log\left(\frac{V_{0}}{\epsilon}\right) + \frac{L}{\mu}\left(\frac{L}{\epsilon}\right)^{\frac{q-\kappa}{\kappa}} + \frac{\sigma}{\mu} \left(\frac{\sigma}{\epsilon}\right)^{q-1}\right).
\]

Also the additional high probability (i.e. the gap not exceeding more than $2\epsilon$ with probability $1-\delta$) cost for $\delta$ small enough is :

\[
O_{q,\kappa}\left( \Big(\frac{\ln(2/\delta) \sigma R}{\epsilon} \Big)^2+ \frac{\sigma}{\mu} \left(\frac{\sigma}{\epsilon} \ln(2/\delta) \right)^{q-1}\right).
\]

\subsection{Accelerated Method}

We propose now an accelerated counterpart of the NACSMD algorithm, that we will call the {\em accelerated composite stochastic mirror-descent} (ACSMD) method. For this, we follow the approach from \cite{GL12II} which attains acceleration by maintaining two sequences of averaged points, one for querying the stochastic first order oracle for $F$, and another for attaining the faster convergence. In the deterministic setting, this algorithm is comparable to the AGD+ algorithm used in \cite{DG21} to obtain acceleration in complementary composite settings. However, there is another distinction, as our method follows a mirror-descent style update, rather than the dual averaging approach pursued in \cite{DG21}. Comparing both methods, our algorithm has a more flexible choice of step-sizes as we will see in the Step-Size Schedule \ref{configuration:stepsize_acc}, not only there are multiple solutions, but we also show that they all attain optimal rates; in particular, for different polynomial step-size schedules, their accuracy only differs by a constant factor. 

\begin{algorithm}
\caption{Accelerated Composite Stochastic Mirror-Descent (ACSMD)}
\label{Algo ACSMD}
\begin{algorithmic}
\Require Iterations $T \geq 0$, starting point $x_1 \in \mathcal{X}$, step-sizes $(\alpha_t,\gamma_t)$
\For{$ 1 \leq t \leq T$}
\State Compute $x_{t}^{md} = \frac{A_{t-1}}{A_{t}} x_{t}^{ag} + \frac{\alpha_{t}}{A_{t}} x_{t}$

\State Compute $x_{t+1} = \arg\min_{x\in {\cal X}} \{\alpha_t [\langle G(x_{t}^{md}, \xi_t) , x \rangle + H(x)] + \gamma_t D^{H}(x,x_{t}) \}$

\State Compute $x_{t+1}^{ag} = \frac{A_{t-1}}{A_{t}} x_{t}^{ag} + \frac{\alpha_{t}}{A_{t}} x_{t+1}$

\EndFor
\State Output $x_{T+1}^{ag}$. 
\end{algorithmic}
\end{algorithm}

\begin{configuration} \label{configuration:stepsize_acc}
Let $\{\alpha_t, \gamma_t\}_{t \geq 1}$ be two sequences such that:
\begin{equation}
    \begin{aligned}
    \begin{cases}
    \alpha_t & \geq \gamma_{t+1} - \gamma_t \\
    \gamma_t & \geq \frac{2M}{\mu} \frac{\alpha_{t}^q}{A_{t}^{q-1}}.
    \end{cases}
    \end{aligned}
\end{equation}
\end{configuration}

For two chosen sequences of step-sizes $\{\alpha_t, \gamma_t\}$, we have a general convergence result.

\begin{theorem} \label{thm:ACSMD_general}
Suppose Algorithm \ref{Algo ACSMD} runs under the step-size schedule \ref{configuration:stepsize_acc}. 
Then, if $q > \kappa$ and if we denote $A_T := \sum_{t=1}^{T} \alpha_t$, for all $T \geq 1$, $x \in \mathcal{X}$:
\begin{multline}
 A_T [ \Psi(x_{T+1}^{ag})- \Psi(x)]  +\gamma_T D^{H}(x,x_{T+1})
\leq \gamma_1 D^{H}(x,x_{1}) \\ 
+ \sum_{t=1}^{T} \alpha_t \langle \Delta_{t}(x_t) , x -x_{t}  \rangle +  \sum_{t=1}^{T} \frac{2\| \Delta_{t}(x_t)\|_{\ast}^{p}}{p \mu^{p/q}} \left(\frac{\alpha_{t}^{q}}{\gamma_t}\right)^{p/q} + L \sum_{t=1}^{T} A_t \left(\frac{2M \alpha_{t}^{q}}{\mu A_{t}^{q-1} \gamma_t}\right)^{\frac{1}{r}}.    
\end{multline}
\end{theorem}

We notice that acceleration factors appear in the last term with $\frac{\alpha_t^q}{A_t^{q-1}\gamma_t }$, which represents the geometric gap.

\begin{proof}
Applying the proximal lemma (Lemma \ref{lem:prox}) to $u=x$, $y=x_t$, $u^{\star}=x_{t+1}$, $f(\cdot) = \alpha_t [l_{\psi}(\cdot,x_{t}^{md}) + \langle \Delta(x_{t}^{md}),\cdot \rangle]$, and $\nu(\cdot)=\gamma_t H(\cdot)$, 
we have:
\begin{align*}
& \alpha_t l_{\Psi}(x_{t+1},x_{t}^{md}) + \gamma_t D^{H}( x_{t+1}, x_t) + \alpha_t D^{H}(x,x_{t+1}) \\
\leq & \alpha_t l_{\Psi}(x,x_{t}^{md}) + \gamma_t [D^{H}(x, x_{t}) - D^{H}(x,x_{t+1})] - \alpha_t \langle \Delta(x_{t}^{md}),x - x_{t+1} \rangle \\
\leq & \alpha_t \Psi(x) + \gamma_t [D^{H}(x, x_{t}) - D^{H}(x,x_{t+1})] - \alpha_t \langle \Delta(x_{t}^{md}),x - x_{t+1} \rangle,
\end{align*}
where we used eqn.~\eqref{proximal gradient inequality} in the last line. Then:
\begin{multline}\label{eq:loss}
\alpha_t \left[ l_{\Psi}(x_{t+1},x_{t}^{md}) -\Psi(x) \right]  + \alpha_t D^{H}(x,x_{t+1})  \\
\leq   \gamma_t [D^{H}(x, x_{t}) - D^{H}(x,x_{t+1})] - \gamma_t D^{H}( x_{t+1}, x_t) 
+\alpha_t \langle \Delta(x_{t}^{md}),x - x_{t+1} \rangle.  
\end{multline}

Next, we use the smoothness of $F$, and the convexity of both $F$ and $H$. Let $\delta_t > 0$ to be determined later: by Lemma \ref{Lemma inexact gradient},
\begin{eqnarray*}
&&\Psi(x_{t+1}^{ag}) = F(x_{t+1}^{ag}) + H(x_{t+1}^{ag}) \\
&\leq& F(x_{t}^{md}) + \langle \nabla F(x_{t}^{md}) , x_{t+1}^{ag} -x_{t}^{md} \rangle + \frac{M}{q\delta_{t}^r} \| x_{t+1}^{ag} -x_{t}^{md}  \|^q +  L \delta_t  + H(x_{t+1}^{ag}) \\
&=& \textstyle F(x_{t}^{md}) + \langle \nabla F(x_{t}^{md}) ,\frac{A_{t-1}}{A_{t}} x_{t}^{ag} + \frac{\alpha_{t}}{A_{t}} x_{t+1} -x_{t}^{md} \rangle + \frac{M}{q\delta_{t}^r} \frac{\alpha_{t}^q}{A_{t}^q} \| x_{t+1} -x_t  \|^q + L \delta_t + H(x_{t+1}^{ag}) \\
&\leq& \frac{A_{t-1}}{A_{t}} \underbrace{\left[F(x_{t}^{md}) +   \langle \nabla F(x_{t}^{md}) , x_{t}^{ag} -x_{t}^{md}  \rangle + H(x_{t}^{ag}) \right]}_{\leq \Psi (x_{t}^{ag}) }  + L \delta_t  \\
&&+ \frac{\alpha_{t}}{A_{t}} \underbrace{ \left[ F(x_{t}^{md}) + \langle \nabla F(x_{t}^{md}) ,  x_{t+1} -x_{t}^{md}\rangle+  H(x_{t+1}) \right]}_{\tstrut{1.5ex} \leq  l_{\Psi}(x_{t+1},x_{t}^{md}) }   + \frac{M}{q\delta_{t}^r} \frac{\alpha_{t}^q}{A_{t}^q} \| x_{t+1} -x_t  \|^q, 
\end{eqnarray*}
where in the last inequality we used the convexity of $H$, specifically 
$H(x_{t+1}^{ag}) \leq \frac{A_{t-1}}{A_{t}} H(x_{t}^{ag}) + \frac{\alpha_{t}}{A_{t}} H(x_{t+1})$.  
Now, using~\eqref{eq:loss}, we have for all $x$: 
\begin{multline*}
A_t \left[ \Psi(x_{t+1}^{ag}) - \Psi(x) \right] \\
\leq   A_{t-1} \left[ \Psi(x_{t}^{ag}) - \Psi(x) \right]  + \alpha_{t} \left[ l_{\Psi}(x_{t+1},x_{t}^{md}) -\Psi(x) \right] + \frac{M}{q\delta_{t}^r} \frac{\alpha_{t}^q}{A_{t}^{q-1}} \| x_{t+1} -x_t  \|^q   + L A_t \delta_t \\
\leq  A_{t-1} \left[ \Psi(x_{t}^{ag}) - \Psi(x) \right]  -   \alpha_t D^{H}(x,x_{t+1}) + \gamma_t [D^{H}(x, x_{t}) - D^{H}(x,x_{t+1})]  \\ 
- \gamma_t D^{H}( x_{t+1}, x_t) + \alpha_t \langle \Delta(x_{t}^{md}),x - x_{t+1} \rangle + \frac{M}{q\delta_{t}^r} \frac{\alpha_{t}^q}{A_{t}^{q-1}} \| x_{t+1} -x_t  \|^q   + L A_t \delta_t.
\end{multline*}

Let now $\delta_t^{r} = \frac{2M \alpha_{t}^{q}}{\mu A_{t}^{q-1} \gamma_t}  $. From the Young inequality, we have:

\begin{equation*}
\begin{aligned}
& \alpha_{t} \langle \Delta(x_{t}^{md}),x_t - x_{t+1} \rangle \\
\leq & \frac{ \|\alpha_{t}\Delta_{t}(x_{t}^{md})\|_{\ast}^{p} }{p \mu (\gamma_t -\frac{M}{\mu \delta_{t}^r} \frac{\alpha_{t}^{q}}{A_{t}^{q-1}})^{p/q}}  + \frac{\mu \gamma_t}{q}\| x_t -x_{t+1}\|^q -  \frac{M}{q\delta_{t}^r} \frac{\alpha_{t}^q}{A_{t}^{q-1}} \| x_{t+1} -x_t \|^q \\
=& \frac{\| \alpha_t \Delta_{t}(x_t)\|_{\ast}^{p}}{p  (\frac{\mu \gamma_t}{2})^{p/q}}  + \frac{\mu \gamma_t}{q}\| x_t -x_{t+1}\|^q -  \frac{M}{q\delta_{t}^r} \frac{\alpha_{t}^q}{A_{t}^{q-1}} \| x_{t+1} -x_t \|^q .
\end{aligned}
\end{equation*}

Combining everything, we have: 
\begin{equation}
\begin{array}{rcl}
& A_t \left[\Psi(x_{t+1}^{ag}) - \Psi(x)\right] -  A_{t-1} \left[ \Psi(x_{t}^{ag}) - \Psi(x) \right]   \\
\leq &  \gamma_t [D^{H}(x, x_{t}) - D^{H}(x,x_{t+1})]  + \frac{2\| \Delta_{t}(x_t)\|_{\ast}^{p}}{p \mu^{p/q}} \left(\frac{\alpha_{t}^{q}}{\gamma_t}\right)^{p/q} \\
&-   \alpha_t D^{H}(x,x_{t+1}) + \alpha_t \langle \Delta(x_{t}^{md}) , x -x_{t}  \rangle + L A_t \left(\frac{2M \alpha_{t}^{q}}{\mu A_{t}^{q-1} \gamma_t}\right)^{\frac{1}{r}}.
\end{array}
\end{equation}

Adding all of our terms from $t=1$ to $t=T$, 
we obtain
\begin{align*}
&  A_T \left[\Psi(x_{T+1}^{ag})- \Psi(x)\right]  +  \gamma_{T} D^{H}(x, x_{T+1}) \\
\leq &  \gamma_1 D^{H}(x,x_{1}) + \sum_{t=1}^{T} \alpha_t \langle \Delta(x_{t}^{md}) , x -x_{t}  \rangle + \sum_{t=1}^{T} \frac{2\| \Delta_{t}(x_t)\|_{\ast}^{p}}{p \mu^{p/q}}  \left(\frac{\alpha_{t}^{q}}{\gamma_t}\right)^{p/q}  \\
+ & \sum_{t=1}^{T} L A_t \left(\frac{2M \alpha_{t}^{q}}{\mu A_{t}^{q-1} \gamma_t}\right)^{\frac{1}{r}} + \sum_{t=1}^{T-1} \left( \gamma_{t+1} - \gamma_{t}-\alpha_t \right) D^{H}(x,x_{t+1}).
\end{align*}

Due to the step-size schedule \ref{configuration:stepsize_acc}, the last term is nonpositive, thus upper bounding it by zero proves the result.
\end{proof}

\paragraph{Polynomial step-sizes under Schedule \ref{configuration:stepsize_acc}} As before, there is a family of step-sizes that increases polynomially with $t$:
\begin{equation}\label{ACSMD_stepsize}
    \begin{aligned}
    \begin{cases}
     m & = \max(\frac{q}{r} - 2, \frac{2-q}{q-1}) \\
     \alpha_t & = (t+ [2(m+1)\frac{M}{\mu}]^{1/q} +1)^m \\
     \gamma_t & = \frac{1}{m+1}(t+  [2(m+1)\frac{M}{\mu}]^{1/q} )^{m+1}. \\
    \end{cases}
    \end{aligned}
\end{equation}

With these parameters, we can state an expected excess risk bound for our accelerated algorithm as follows. 

\begin{theorem} \label{thm:ACSMD}
Under Assumption \ref{assump:gradient_noise} and choosing step-sizes as in \eqref{ACSMD_stepsize}, if $q > \kappa$, we have for all $T \geq 1$:
\begin{equation*}
\mathbb{E} \left[\Psi(x_{T+1}^{ag}) - \Psi^{\star} + D^{H}(x_{\star},x_{T+1}) \right] \leq O_{q,\kappa} \left( \frac{L^{\frac{m+1}{q}} V_0 }{(\mu^{\frac{1}{q}} T)^{m+1}}  +   \left(\frac{L}{\mu}\right)^{\frac{1}{r}} \frac{L}{T^{\frac{q-r}{r}}} + \frac{\sigma^{p}}{(\mu T)^{\frac{p}{q}}} \right),
\end{equation*}
where $V_0 = D^{H}(x_{\star},x_1)$ and $O_{q,\kappa}$ omits absolute constants that depend on $q,\kappa$. Moreover, if $\mathcal{X}$ is bounded with  diameter $R$ and if we note $\mathfrak{M}_{T}(\Psi)$ the previous upper bound for $\Psi$, under Assumption \ref{assumption:mgf_surrogate}, we have for all $\Omega >0$:
\begin{equation} \label{eqn:Accelerated}
\begin{aligned} 
&\mathbb{P}\left[\Psi(x_{T+1}^{ag}) - \Psi^{\star} \geq \mathfrak{M}_{T}(\Psi) + O_{q,\kappa} \left(\frac{\Omega \sigma R}{\sqrt{T}} + \frac{\Omega \sigma^p}{(\mu T)^{p/q}} \right)\right]\\
\leq &  \begin{cases}
\exp(-\Omega) + \exp\Big\{-\frac{1}{4}   \Omega^2 \Big\} \quad & \textit{if} \quad  \Omega \lesssim \sqrt{T} \\
\exp(-\Omega) + \exp\Big\{-\frac{1}{p}    (T^{\frac12-\frac1p}\Omega)^p \Big\} \quad & \textit{if} \quad \Omega \gtrsim \sqrt{T}. 
\end{cases}
\end{aligned}
\end{equation}
with $O_{q,\kappa}$ omits another absolute constant that depend on $q,\kappa$.
\end{theorem}

\begin{remark}
Similarly to Theorem \ref{thm:NACSMD}, if $q=\kappa=2$, then the inexact gradient trick is unnecessary, and with minor adaptations to the proof, the rate below follows
 \begin{equation}
\mathbb{E} \left[\Psi(x_{T+1}^{ag}) - \Psi^{\star} + D^{H}(x_{\star},x_{T+1}) \right] \leq O 
\left(\frac{L}{\mu}\frac{V_0 }{T^{2}} + \frac{\sigma^{2}}{\mu T} \right) =:\mathfrak{M}_{T}'(\Psi).
\end{equation}
For the high probability bound, an analog of eqn.~\eqref{eqn:Accelerated} holds, where $\mathfrak{M}_{T}(\Psi)$ is replaced by $\mathfrak{M}_{T}'(\Psi)$. The details are left to the reader.
\end{remark}

The in-expectation guarantee follows directly from Theorem \ref{thm:ACSMD_general}, hence its proof is omitted. 
The proof of the concentration bound is deferred to Appendix \ref{Proof Concentration}. Moreover, we combine the previous result with the restarting algorithm (Algorithm \ref{Restarting Algorithm}) to reduce the initialization error with $V_0 = D^{H}(x_{\star},x_1)$. As before, that includes new coefficients for the initialisation term $ K_1 \propto (L/\mu)^{\frac{n+1}{q}}, \alpha_1 = n+1 $ and geometric gap $K_2 \propto \Big(\frac{L}{\mu}\Big)^\frac1r, \alpha_2 = \frac{q-r}{r}$, but the stochastic terms stay the same $K_3 Z_t \propto \frac{\|\Delta(x_t^{md})\|_\ast^p}{\mu^{p/q}}, K_4 W_t \propto \alpha_t \langle \Delta(x_t^{md}), x_{\star} - x_t\rangle , \alpha_3 = p/q, \alpha_4 = 1 $; this reflects the inherent nature of this term in the accuracy. The final complexity for the in-expectation guarantee becomes:
\[
O_{q,\kappa} \left(\left(\frac{L}{\mu}\right)^{\frac{1}{q}} \log\left(\frac{ V_{0}}{\epsilon}\right) + \left(\frac{L}{\mu}\right)^{\frac{\kappa}{q\kappa+q-\kappa}} \left(\frac{L}{\epsilon}\right)^{\frac{q-\kappa}{q\kappa-q+\kappa}} + \frac{\sigma}{\mu} \left(\frac{\sigma}{\epsilon}\right)^{q-1}\right).
\]

And the additional high probability (i.e. the gap not exceeding more than $2\epsilon$ with probability $1-\delta$) cost for $\delta$ small enough is :

\[
O_{q,\kappa}\left( \Big(\frac{\ln(2/\delta) \sigma R}{\epsilon} \Big)^2+ \frac{\sigma}{\mu} \left(\frac{\sigma}{\epsilon} \ln(2/\delta) \right)^{q-1}\right).
\]

We notice that even if we apply our algorithm in the deterministic case $\sigma = 0$, the convergence rate is sharper than the one in \cite{DG21}. And we will see in the next subsection that our current stochastic bound $\frac{\sigma}{\mu} \left(\frac{\sigma}{\epsilon}\right)^{q-1}$ is optimal when we only have an estimator of the gradient with $p$-th finite moment. 

\section{Lower Complexity Bounds}

To show the near-optimality of our algorithms, we provide matching lower bounds in all parameter settings. These lower bounds are obtained by combining (deterministic) oracle complexity bounds for complementary composite minimization from past work, together with lower bounds applicable to stochastic oracles that satisfy Assumption \ref{assump:gradient_noise}. Due to space limitations, we do not provide a detailed description of the oracle model, recommending as references \cite{NY:1983,Nemirovski:1995}. In general terms, this model captures the interaction of an algorithm with an instance only through queries (corresponding to feasible solutions) whose oracle answers depend only locally on the function (e.g., its gradient). After this interaction, {\em the algorithm must commit to a candidate solution} (this choice can even be randomized, but it must be based exclusively on the information collected so far), and the efficiency of the method is determined by the number of queries it performs. Further, the algorithm must provide an output (which either in expectation or with high probability) has suboptimality gap at most $\varepsilon>0.$

\subsection{Deterministic Lower Complexity Bound}

The oracle complexity of deterministic complementary composite minimization was first studied in \cite{DG21}. We refer the reader to this work for a more precise description of the oracle model, which -- in a nutshell -- assumes exact first-order oracle access to $F$ and full access to $H$.

\begin{theorem}[\cite{DG21}] \label{thm:deterministic_LB}
Consider the space $\ell_q^d=(\mathbb{R},\|\cdot\|_q)$, where $2\leq q <\infty$. Then, the oracle complexity of complementary composite minimization problems where $F$ is $(L,\kappa)$-weakly smooth, $H$ is $(\mu,q)$-uniformly convex and ${\cal X}$ has $\|\cdot\|$-diameter at least $D$, is lower bounded by:

\[
\left\{
\begin{array}{lcl}
\sqrt{\frac{L}{2\mu}}-7 & & \mbox{ if } q=\kappa=2, \varepsilon <2\sqrt{2\mu L}D^2 \\
\frac{C}{\min\{q,\ln d\}^{2(\kappa-1)}}\Big( \big(\frac{L}{\mu}\big)^{\kappa} \big(\frac{L}{\varepsilon}\big)^{q-\kappa} \Big)^{\frac{1}{\kappa q+\kappa-q}} & & \mbox{ if } 1\leq \kappa <q, 2\leq q\leq \infty,\mbox{ and } \mu\geq \tilde\mu,
\end{array}
\right.
\]

where $C=\Big(\big(\frac{q-1}{q}\big)^{\kappa(q-1)}2^{\frac{(q-\kappa)(1-2q)+(\kappa-1)q(2q-3)}{(q-1)}}\Big)^{\frac{1}{\kappa q+\kappa-q}}$ is bounded below by an absolute constant, and
\[ \tilde\mu = C^{\prime} \max\left\{ \min\{q,\ln d\}^3\Big(\frac{\varepsilon^{\kappa}}{LD}\Big)^{\frac{1}{\kappa-1}}, \min\{q,\ln d\}^5\Big( \frac{\varepsilon^q}{L^{(q+1)}D^{\frac{(q-1)(\kappa q+\kappa-q)}{(\kappa-1)}}} \Big)^{\frac{\kappa-1}{\kappa q+1-q}} \right\}, \]
where $C>0$ is a universal constant.
\end{theorem}

This result is applicable to all settings of $2\leq q<\infty$ and $1<\kappa\leq 2$, for arbitrary choices of $L>0$, however the lower bound only applies to sufficiently large values of $\mu$. This limitation is inherent, as when the uniform convexity parameter becomes sufficiently small better complexity rates are obtained by non-uniformly convex stochastic optimization. Note moreover that the sufficiently large condition -- given by $\tilde\mu$ --  scales polynomially with the target accuracy, hence the restriction is mild.

In the case $q=\kappa=2$, the lower bound of the theorem, $\Omega(\sqrt{L/\mu})$, is nearly tight in the case $\sigma=0$, due to the upper bound $O(\sqrt{L\log(V_0/\epsilon)/\mu})$; an analogous rate was obtained in \cite{DG21}. We defer the case $\sigma>0$ to the next subsection. On the other hand, when $\kappa<q$ we obtain a polynomial lower bound $\tilde\Omega\big(\big(\big(\frac{L}{\mu}\big)^{\kappa} \big(\frac{L}{\varepsilon}\big)^{q-\kappa} \big)^{\frac{1}{\kappa q+\kappa-q}}\big)$. Our upper bound in this setting when $\sigma=0$ is $O\big(\big(\frac{L}{\mu}\big)^{1/q}\log\big(\frac{V_0}{\epsilon}\big)+\big(\big(\frac{L}{\mu}\big)^{\kappa}\big(\frac{L}{\epsilon}\big)^{q-\kappa}\big)^{\frac{1}{q\kappa-q+\kappa}} \big)$, hence the lower bound is nearly-optimal up to a poly-logarithmic additive term. Regarding this gap, we remark that our result is a refinement of results in \cite{DG21}, where the logarithmic term appears multiplicatively in the second term of our upper bound.

\subsection{Stochastic Lower Complexity Bound}

Our stochastic lower bound is inspired by a very classical argument \cite{NY:1983}, which we extend to the uniformly convex setting, as well as extending to arbitrary moment parameter $\sigma>0$. 

\begin{theorem}\label{thm:stoch_LB}
Consider the class of problems \eqref{eqn:comp_min} where $2\leq q<\infty$, $1<\kappa\leq 2$, and $\varepsilon,\mu, \sigma>0$ satisfying
\begin{equation} \label{eqn:assump_stoch_LB}
\varepsilon \leq \frac{1}{2p}\frac{\sigma^p}{\mu^{p-1}}.
\end{equation}
Under Assumptions \ref{assump:gradient_noise} and \ref{assump:subproblem}, any algorithm for this problem class is such that, for any $0<\gamma<1$, with probability $1-\gamma$ it fails on achieving accuracy $\varepsilon$ after \[T\leq\frac{1}{2p^{q-1}}\frac{\sigma}{\mu}\Big(\frac{\sigma}{\epsilon}\Big)^{q-1}\ln\big(\frac{1}{1-\gamma}\big).\]
many queries to a certain stochastic first-order oracle.
\end{theorem}

First, we give an overview of the approach. We will consider a 1-dimensional instance of the form $F_{\nu}(x)=\mathbb{E}_b[b\nu Cx]$, and $H(x)=(\mu/q)|x|^q$, where  
$\nu=\pm1$ can be adversarially selected, and $b$ is a binary random variable that takes the value 0 with probability $1-s$ and the value $1/s$ with probability $s$; we will choose $s$ very small. Notice that $\nu$ tilts the optimal solution to the right or left of the origin, and we will show that in fact learning $\nu$ is necessary and sufficient to accurately minimize the composite objective. The key idea is that the oracle in this case corresponds to samples from $b$, which only rarely provide $b\neq 0$; in this case, we are unable to learn the parameter $\nu$, and essentially cannot make any progress. Hence, controlling the probability that $T$ samples from $b$ are all zero suffice to assert that the algorithm is unlikely to succeed in terms of objective function value.

\begin{proof}
Let $0<s<1$ and $b$ be a random variable that takes value 0 w.p.~$1-s$ and value $1/s$ w.p.~$s$. Clearly, $\mathbb{E}[b]=1$. Given $\nu\in\pm1$ and $C>0$ to be determined, consider the following functions:
\begin{equation*}
f_{\nu}(x,b) =\nu b Cx \qquad
H(x) = \frac{\mu}{q}|x|^q.
\end{equation*}

Thus, the objective $F_{\nu}(x)=\mathbb{E}_b[f_{\nu}(x,b)] +H(x)$ satisfies the complementary composite structure: namely, the objective is composed by a $(L,\kappa)$-smooth function (where in fact $L=0$) plus a $(\mu,q)$-uniformly convex function.

In what follows, we will determine values of $s$ and $C$ such that the $\varepsilon$-level sets of $F_{+1}$ and $F_{-1}$ are disjoint, and that Assumption \ref{assump:gradient_noise} is satisfied

\begin{itemize}
    \item {\em Level set disjointness:} By the optimality conditions for $F_{\nu}$, it is easy to see the optimal solution is $x_{\nu}^{\ast}:=-\nu\big( \frac{C}{\mu}\big)^{\frac{1}{q-1}}$, and hence 
    \[ F_{\nu}(0)-F_{\nu}(x_{\nu}^{\ast})=\frac{1}{p}\Big(\frac{C^q}{\mu}\Big)^{\frac{1}{q-1}}. \]
    Hence, the condition below suffices to have the property that the $\varepsilon$-level sets of $F_1$ and $F_{-1}$ are disjoint
     \begin{equation} \label{eqn:lev_set_disjoint} 
     C= \mu^{1/q}(\varepsilon p)^{1/p}. 
     \end{equation}
    \item {\em Centered moment bounds.} Let us compute the $p$-th moment of our stochastic oracle:
    \begin{multline*}
    \mathbb{E}_b\|\nabla f_{\nu}(x,b)-\nabla F_{\nu}(x)\|_{\ast}^p \\
    = (1-s)C^p+s\big(\frac{C}{s}-C\big)^p =C^p(1-s)\Big(1+\big(\frac{1-s}{s}\big)^{p-1} \Big).
    \end{multline*}
    We want to impose that this moment is upper bounded by $\sigma^p$, which is equivalent to:
    \[\frac{1-s}{s} \leq \Big(\Big( \frac{\sigma}{C} \Big)^p \frac{1}{1-s}-1 \Big)^{\frac{1}{p-1}}.\]
    Notice that by \eqref{eqn:assump_stoch_LB}, we have that $(\sigma/C)^p\geq 2$, which in turn implies $\frac{(\sigma/C)^p}{1-s}-1\geq \frac{(\sigma/C)^p}{2(1-s)}$, so it suffices that
    
    \begin{eqnarray*}
    \frac{1-s}{s} \leq \Big(\frac12\Big( \frac{\sigma}{C} \Big)^p \frac{1}{1-s}\Big)^{\frac{1}{p-1}}
    \quad\Leftrightarrow\quad
    \frac{s^{p-1}}{(1-s)^p}\geq \Big( \frac{2p\mu^{p-1}\varepsilon}{\sigma^p} \Big)^{\frac{1}{p-1}}.
    \end{eqnarray*}
    Finally, notice the left hand side is a monotonically increasing function of $s$, and that when $s\to 0$ it converges to 0, whereas when $s\to 1$ it diverges to $+\infty$. Hence, there exists a unique choice of $s$ such that equality is satisfied. From now on, we make this choice of $s$.
\end{itemize}
The proof is concluded by noticing that the probability that $T$ samples $b_1,\ldots, b_T$ provide $\nabla f_{\nu}(x,b_t)=0$ is $(1-s)^T$. Notice that under this event, the algorithm has collected no information about $\nu$. Therefore, if we choose $\nu=\pm1$ uniformly at random, the expected suboptimality of the algorithm will be at least $\varepsilon$. We finally lower bound the probability of the event above, for which we use the elementary inequality $\ln(1-x)\geq -\frac{x}{1-x}$, for $0<x<1$:
\begin{eqnarray*}
(1-s)^T&=&\exp\big\{T\ln(1-s) \big\} \geq \exp\big\{-\frac{sT}{1-s} \big\}.
\end{eqnarray*}
Notice now that for our choice of $s$, we have:
\[ \frac{1-s}{s} = \Big(\frac12\Big( \frac{\sigma}{C} \Big)^p \frac{1}{1-s}\Big)^{\frac{1}{p-1}} \geq \Big(\frac12\Big( \frac{\sigma}{C} \Big)^p \Big)^{\frac{1}{p-1}} =\frac{1}{2^{q/p}}p^{q-1}\frac{\sigma^q}{\mu \varepsilon^{q-1}}.\]
Finally, making the choice of $T=\frac{1}{2^{q/p}p^{q-1}}\frac{\sigma^q}{\mu \varepsilon^{q-1}}\ln\big(\frac{1}{1-\gamma}\big)$ shows the probability is at least $1-\gamma$, proving the result.
\end{proof}

To conclude this Section, we briefly discuss the consequences of the separate lower bounds proved in Theorems \ref{thm:deterministic_LB} and \ref{thm:stoch_LB}; in particular: Do they imply a lower bound given by the sum of the two? The answer is yes (possibly with a degradation by an absolute constant factor), and the argument is the following: Consider an adversary which first tosses a fair coin, and based on its outcome it selects either the family of instances from the proof of Theorem \ref{thm:deterministic_LB}, or alternatively it selects the family of instances from the proof of Theorem \ref{thm:stoch_LB}. Then, for any algorithm, its expected running time against this random instance must be proportional to the sum of the two lower bounds. Furthermore, if the algorithm is deterministic, then we can derandomize this choice, making the lower bound to hold with high-probability.

\section{Numerical results}

We apply our methods to the generalized ridge regression problem, as presented in eqn.~\eqref{eqn:bridge_reg},
\begin{equation}
\min_{x\in\mathbb{R}^d} \mathbb{E}_{(a,b)}[(a^{\top} x-b)^2]+\mu\|x\|_q^q.
 \end{equation}

We generate synthetic data from a uniform distribution $\mathcal{U}$ and Gaussian noise:
\begin{equation}
    \begin{cases}
    a  \propto 
    \mathcal{U}([-1,1]^d) \\
    b =a^{\top}x_{\star}+\xi, \qquad \xi \propto \mathcal{N}(0,\sigma_b).
    \end{cases}
\end{equation}

We know that the loss function $F$ is smooth and strongly convex with respect to the $\ell_q$ norm, but the condition number 
depends on the dimension: 

\begin{equation*}
    \begin{cases}
    \nabla F(x) & = - 2  \mathbb{E}_{(a,b)} \left[ (a^{\top}(x_\star - x) + \xi ) a \right] =\frac{2}{3} (x - x_\star) \\
    D^F(x_\star,x) & = \mathbb{E}_{(a,b)} \left[ (a^{\top}(x - x_\star))^2 \right] = \frac{1}{3}\| x - x_\star \|_{2}^2.  
    \end{cases}
\end{equation*}
Dimension dependence arises from the following chain of inequalities (each of which can be tight in the worst case):
\[\frac{1}{3}\| x - x_\star \|_{q}^2 \leq  D^F(x_{\star},x) \leq \frac{d^{1 - \frac{2}{q}}}{3}\| x - x_\star \|_{q}^2.  \]
Hence, the condition number of $F$ can be upper bounded by $d^{1-\frac2q}$.

On the other hand, the regularizer is $(\mu,q)$-uniformly convex with respect to the $\ell_q$ norm (see, e.g., \cite{Ball:1994}). To check the performance of the algorithms, we choose the following setting. Below we denote by $\mu_F$ the strong convexity parameter of $F$: 

\begin{equation*}
    \begin{cases}
    L & = \frac{2}{3} d^{1 - \frac{2}{q}}  \\
    \mu_F & = \frac{2}{3} \\
    \sigma & = d^{2/p} \sigma_b^2 + 2 d^2 R^2 
    \end{cases}
\end{equation*}

\begin{figure}[!h]
    \centering
    \includegraphics[width=0.7\linewidth]{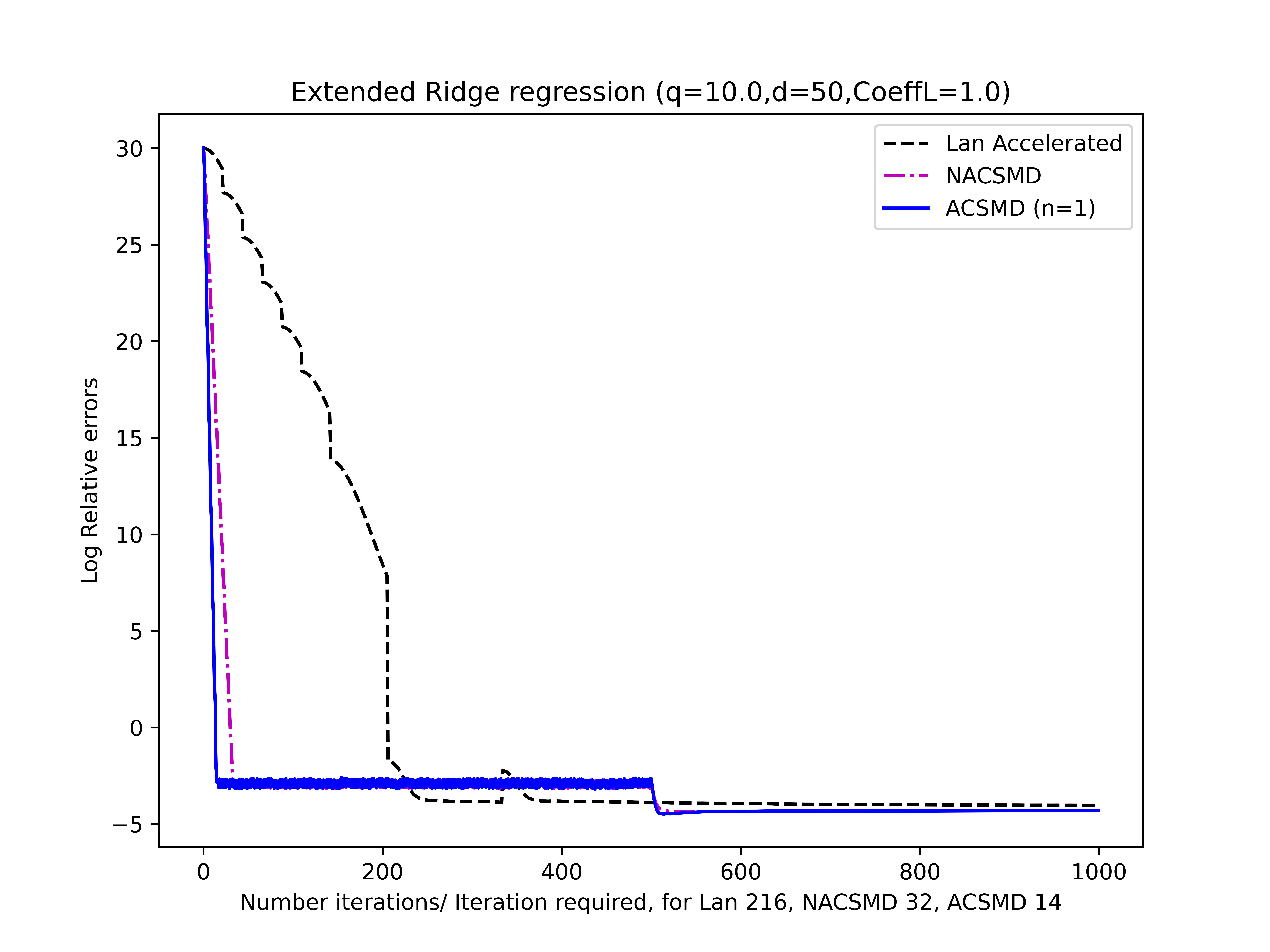}
    \caption{Performance comparison between Algorithms \ref{Algo NACSMD}, \ref{Algo ACSMD} and the one suggested in \cite{GL12II} with the restarting scheme \ref{Restarting Algorithm}. We evaluate the decreasing speed of the log relative error through iterations for an extended Ridge regression problem.}
    \label{fig:simulation_1}
\end{figure}

To better evaluate the performance of different algorithms, we first pick $\sigma_b=0.1$ and $\mu=2$, we run multiple simulations with different parameters of $d$ and $L$, then we measure the needed number of iterations to achieve the relative $\epsilon = 0.01$ precision. For the step size, we choose  $\alpha_t$ as a constant for NACSMD and different polynomials for ACSMD.

\begin{table}[ht]
    \centering
    \begin{tabular}{|c | c | c | c | c | c |} 
     \hline
     Iteration required & Lan & NACSMD & ACSMD1 & ACSMD2 & ACSMD3\\ [0.5ex] 
     \hline\hline
     $d$=20 &  91 & 81  & 32  & 20  & 14 \\ 
     \hline
     $d$=50 & 110 & 66 & 26 & 16 & 12\\
     \hline
     $d$=100 & 145  & 82  & 33  & 21  & 15  \\
     \hline
     $d$=200  & 138  & 76  & 26  &  17 & 12  \\ 
     \hline
    \end{tabular}
    \caption{Simulation results with different dimensions.}
    \label{table_d}
\end{table}

\begin{table}[ht]
    \centering
    \begin{tabular}{|c | c | c | c | c | c |} 
     \hline
     Iteration required & Lan & NACSMD & ACSMD1 & ACSMD2 & ACSMD3 \\ [0.5ex] 
     \hline\hline
     $L$ & 110 & 66 & 26 & 16 & 12 \\ 
     \hline
     $2L$ & 149  & 84  & 33  & 21  & 15  \\
     \hline
     $5L$ &  114 & 123  & 26  & 16  & 12 \\
     \hline
     $10L$ &  228 & 266  &  28 & 18  & 13 \\
     \hline
     $20L$ & 457  & >999  & 31  & 20  & 14  \\ 
     \hline
    \end{tabular}
    \caption{Simulation results with overestimated parameter $L$ ($d=50$).}
    \label{table_L}
\end{table}

Now we can briefly analyze the results we have obtained. We emphasize three aspects that illustrate the benefits of our accelerated method.

\paragraph{Step-size schedule} First, we observe that our proposed step-size schedule works well in practice as predicted by our theory. In Section \ref{sec:algorithms} 
we have shown the flexibility of the step-sizes provide convergence. Here, we have used different step-sizes by changing the degree of the interpolating polynomial. More precisely, ACSMD1, ACSMD2 and ACSMD3 have their step-sizes chosen as :
\[ \left(\alpha_t^{ACSMD1},\alpha_t^{ACSMD2},\alpha_t^{ACSMD3}\right) = \left( \Big(t + \big(\frac{L}{\mu}\big)^\frac1q\Big),\Big(t + \big(\frac{L}{\mu}\big)^\frac1q\Big)^{2}, \Big(t + \big(\frac{L}{\mu}\big)^\frac1q\Big)^{3}\right), \] and as we can observe in Figure \ref{fig:simulation_2}, all three algorithms have good convergence rate as predicted. This shows that the flexibility of the method also exists in practice.  

\begin{figure}[!h]
    \centering
    \includegraphics[width=0.7\linewidth]{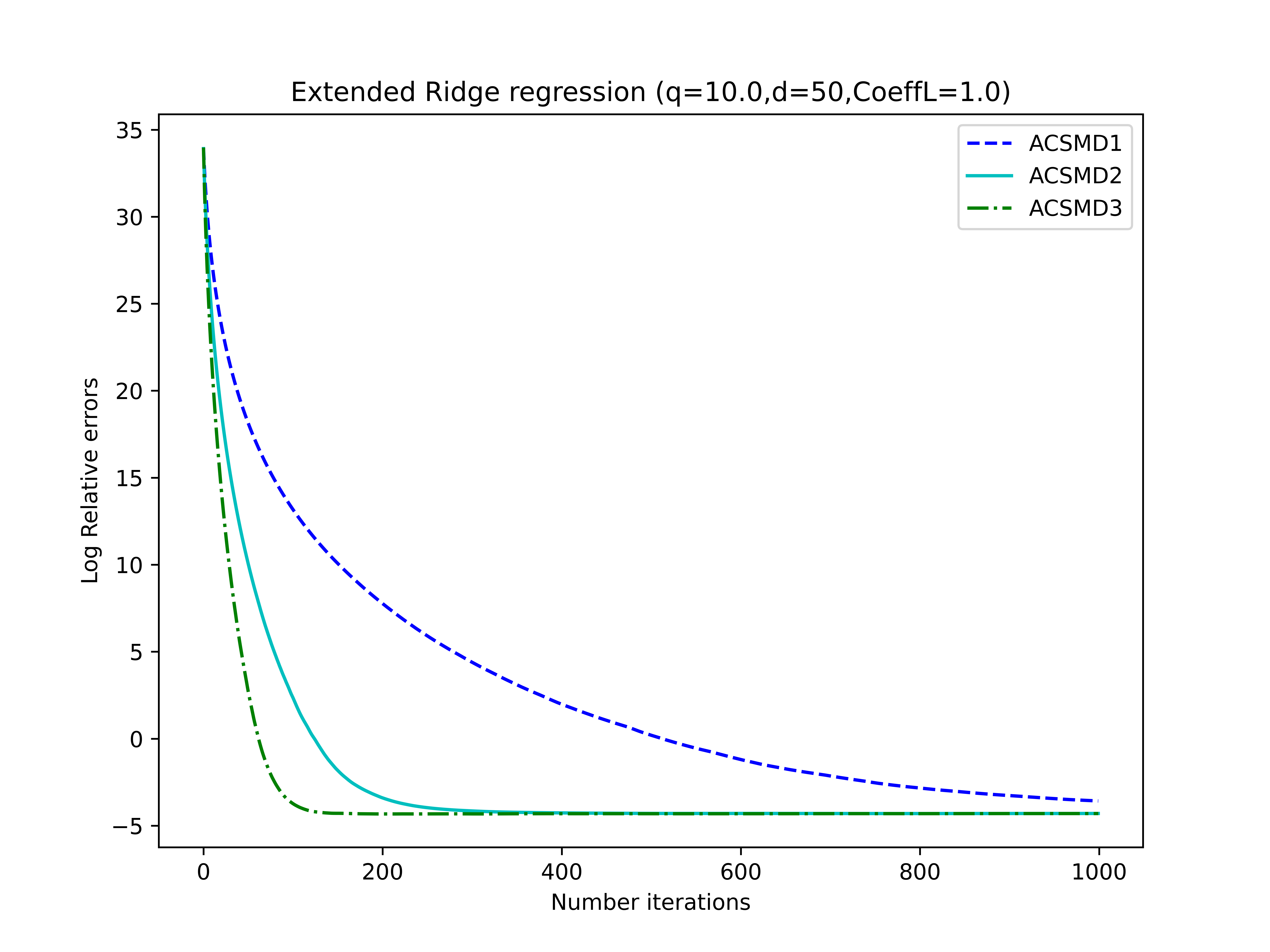}
    \caption{Performance comparison for Algorithm \ref{Algo ACSMD} with different polynomial degrees as described before without the restarting scheme. We evaluate the decreasing speed of the log relative error through iterations for an extended Ridge regression problem.}
    \label{fig:simulation_2}
\end{figure}

\paragraph{Uniformly convex regularizer and condition number} Our algorithms also work for uniformly convex regularizer which is not much covered in the existing literature. In higher dimensional complementary composite setting (Figure \ref{fig:simulation_1}), our algorithms converge faster the ones in \cite{GL12}, designed for the classical composite setting. In the state of the art, for a problem such as \eqref{eqn:bridge_reg}, the condition number ($L/\mu_F$) increases with the dimension, which makes the convergence rate slower. 

We would also like to emphasize the fact that our example is such that $F$ is strongly convex. Notice however that its strong convexity constant will be directly tied to the distributional features of the data, which for ill-posed problems it can lead to arbitrarily small strong convexity parameter. 

On the other hand, even if this quantity is finite, there might not be any available bound on it, because determining this parameter is computationally difficult.
In contrast, our complementary composite setting will enjoy a bounded effective condition number, as long as the regularization term is uniformly convex; moreover, its parameters are typically an algorithmic design choice, thus we do not need to estimate them.

\paragraph{Acceleration benefits} Our last experiments, in Table \ref{table_L}, explores the computational cost of over-estimating the parameter $L$. This is a key concern in practical scenarios, where precisely estimating this parameter is difficult (if not impossible). 
An over-estimation of $L$ would lead to an over-estimation of the condition number which would increase the number of iterations considerably. In our experiments, we observed that our accelerated algorithm has the same complexity when over-estimating $L$ up to a factor $20$, which we attribute to the milder dependence of our methods (particularly, with a polynomial root) in terms of the condition number. This phenomenons is also predicted by the theory and illustrates the importance of acceleration.

\section{Acknowledgments}
Research partially supported by an INRIA Associate Teams grant. 
CG's research was partially supported by FONDECYT 1210362 grant and National Center for Artificial Intelligence CENIA FB210017, Basal ANID.

\nocite{*}
\bibliographystyle{plain} 
\bibliography{references}

\appendix

\section{Example of uniform convexity}
\label{uniform convexity example}

We remind the context of the uniform convexity of $f(x)= \|x\|_q^q$. We  would like to show that, for all $x,y \in \mathbb{R}^d$,  $D_f(x,y) \geq \frac{2^{-\frac{q(q-2)}{q-1}}}{q} \| x - y \|_q^q = \frac{2^{-\frac{q(q-2)}{q-1}}}{q} \sum_{i=1}^d |x_i - y_i|^q$.

By the separability on $f(x) = \frac{1}{q} \sum_{j=1}^d |x_j|^q $ and $ \nabla f(x) = (|x_j|^{q-1} \mbox{sgn}(x_j) )_{j\in[d]}$, we only need to show the result in dimension one, which means that, for all $x,y \in \mathbb{R},$
\[ |x|^q - |y|^q - q y |y|^{q-2} (x-y) \geq \frac{2^{-\frac{q(q-2)}{q-1}}}{q}  |x-y|^q. \]
which is proved in \cite[Proposition 3.2]{Zalinescu:1983}.

\section{Analysis of the restarting algorithm}
\label{annexes:proof restart}

\begin{proof} {\em of Lemma \ref{lemma restarting}.}
The proof is composed into two parts, we will first analyze the output in the $n$ rounds of fixed length $K = \lceil (2K_1)^{\frac{1}{\alpha_1}} \rceil$ after each iteration. Then once the initialisation error(related to $K_1$) is considerably reduced, we will analyse the output after the remaining $T$ iterations to show that the complexity costs for other terms (related to $K_2,K_3,K_4$) have at most doubled. 

For the first $n$ iterations rounds, we notice that the assumption in the lemma \ref{lemma restarting} provide a recursive form for the proximal function. If we call $x_{1}^{k}$ the restarting point that we use for the $k$-th round and $x_{K+1}^{k}$ the output:
\begin{equation*}
\begin{aligned}
 D^{H}(x_{\star},x_{K+1}^{k}) & \leq \frac{K_1 D^{H}(x_{\star},x_1^{k})}{K^{\alpha}} + \frac{K_2}{K^{\alpha_2}} + \frac{K_3}{K^{\alpha_3}} \sum_{t=1}^K Z_{t}^{k} + \frac{K_4}{K^{\alpha_4}} \sum_{t=1}^K W_{t}^{k} \\
D^{H}(x_{\star},x_{K+1}^{k}) & \leq \frac{1}{2}D^{H}(x_{\star},x_{1}^{k})  + \frac{K_2}{K^{\alpha_2}} + \frac{K_3}{K^{\alpha_3}} \sum_{t=1}^K Z_{t}^{k} + \frac{K_4}{K^{\alpha_4}} \sum_{t=1}^K W_{t}^{k}  \\
\end{aligned}
\end{equation*}

with $Z_t^k, W_t^k$ random variables appeared in $k$-th round. We realize that the distance to the optimal solution $D^{H}(x_{\star},x_{K+1}^{k})$ has almost been halved compared to our initial distance $D^{H}(x_{\star},x_{1}^{k})$ and we are paying a constant cost related to $K$. In other words, we have a recursion of the form $Y_{k+1} \leq \frac{1}{2} Y_{k} + C$, where $C>0$ is a constant.

Now we only need to remind that the restarting point is the ending of the previous epoch: $x_{1}^{k}= x_{K+1}^{k-1}$. For each round, we are paying the same constant price, but since the scale is halved each time, the sum of them is converging. For example if $n \geq 2$:

\begin{align*}
D^{H}(x_{\star},x_{K+1}^{n}) & \leq \frac{1}{2}D^{H}(x_{\star},x_{1}^{n-1})  + \frac{K_2}{K^{\alpha_2}} + \frac{K_3}{K^{\alpha_3}} \sum_{t=1}^K Z_{t}^{n} + \frac{K_4}{K^{\alpha_4}} \sum_{t=1}^K W_{t}^{n}  \\
 & \leq \frac{1}{2}\Big( \frac{1}{2}D^{H}(x_{\star},x_{1}^{n-2})  + \frac{K_2}{K^{\alpha_2}} + \frac{K_3}{K^{\alpha_3}} \sum_{t=1}^K Z_{t}^{n-1} + \frac{K_4}{K^{\alpha_4}} \sum_{t=1}^K W_{t}^{n-1}  \Big)   \\
 & + \frac{K_2}{K^{\alpha_2}} + \frac{K_3}{K^{\alpha_3}} \sum_{t=1}^K Z_{t}^{n} + \frac{K_4}{K^{\alpha_4}} \sum_{t=1}^K W_{t}^{n} \\
& = \frac{1}{4} D^{H}(x_{\star},x_{1}^{n-2})  + \Big( 1 + \frac{1}{2} \Big) \frac{K_2}{K^{\alpha_2}}    \\
 & + \frac{K_3}{K^{\alpha_3}} \sum_{t=1}^K \Big( \frac{Z_{t}^{n-1}}{2} + Z_{t}^{n} \Big) + \frac{K_4}{K^{\alpha_4}} \sum_{t=1}^K \Big( \frac{W_{t}^{n-1}}{2} + W_{t}^{n} \Big) 
\end{align*}

Thus, we have by induction for $n\geq 1$: 
\begin{align*}
D^{H}(x_{\star},x_{K+1}^{n}) \leq \frac{K_1 D^{H}(x_{\star},x_1)}{2^n} + \frac{K_2}{K^{\alpha_2}} \sum_{i=0}^{n-1} \frac{1}{2^k}
\\ + \frac{K_3}{K^{\alpha_3}} \sum_{k=1}^n \sum_{t=1}^K  \frac{Z_{t}^k}{2^{n-k}} + \frac{K_4}{K^{\alpha_4}} \sum_{k=1}^n \sum_{t=1}^K  \frac{W_{t}^k}{2^{n-k}}.  
\end{align*} 

The remaining part is to run $T \geq K $ iterations with the new starting point $x_{K+1}^{n}$ with $Y_{T+1}$ the final output:
\begin{equation*}
\begin{aligned}
& \Psi(y_{T+1}) - \Psi^{\star} \\
\leq & \frac{K_1 D^{H}(x_{\star},x_{K+1}^{k})}{T^{\alpha_1}}  + \frac{K_2}{T^{\alpha_2}} + \frac{K_3}{T^{\alpha_3}} \sum_{t=1}^T Z_t + \frac{K_4}{T^{\alpha_4}} \sum_{t=1}^T W_t\\
\leq & \frac{D^{H}(x_{\star},x_1)}{2^{n+1}} + \frac{3K_2}{T^{\alpha_2}} + \frac{K_3}{T^{\alpha_3}} \Big( \sum_{t=1}^T  Z_t +  \frac{T^{\alpha_3 -\alpha_1}}{K^{\alpha_3}} \sum_{k=1}^n \sum_{t=1}^K  \frac{Z_{t}^k}{2^{n-k}} \Big) \\
& + \frac{K_4}{T^{\alpha_4}} \Big( \sum_{t=1}^T  W_t +  \frac{T^{\alpha_4 -\alpha_1}}{K^{\alpha_4}} \sum_{k=1}^n \sum_{t=1}^K  \frac{W_{t}^k}{2^{n-k}} \Big)
\end{aligned}
\end{equation*}

\end{proof}

\section{Concentration inequalities}
\label{app:concentration}


\begin{lemma} 
\label{lemma:mgf_surrogate}
Let $W$ be a random variable that satisfies \eqref{eqn:mgf_real}. Then
\[ 
\mathbb{E}\big[\exp\big\{\lambda W \big\}\big] \leq 
\left\{ 
\begin{array}{ll}
 \exp \Big\{(3 \lambda \sigma R )^2  \Big\} & |\lambda| \leq \frac{1}{2\sigma R}\\
\exp\Big\{(3 |\lambda| \sigma R )^q  \Big\} &|\lambda| > \frac{1}{2\sigma R}.
\end{array}
\right.
\] 


\end{lemma}

\begin{proof}
First, consider the case $|\lambda|\leq 1/[2\sigma R]$.
By Markov's inequality:
\[ \mathbb{P}(|W| \geq \lambda ) = \mathbb{P}\left( \exp \Big\{\frac{|W|^p}{\sigma^p R^p} \Big\} \geq \exp \Big\{\frac{\lambda^p}{\sigma^p R^p}\Big\} \right) \leq 2\exp \Big\{ -\frac{\lambda^p}{\sigma^p R^p} \Big\} \]

Then we can also calculate the moments, for $\alpha \geq 1$:
\begin{equation*}
\begin{aligned}
\mathbb{E} |W|^\alpha & = \int_{0}^{\infty} \mathbb{P}(|W| \geq \lambda ) \alpha \lambda^{\alpha-1} d\lambda \\
& \leq 2 \alpha \int_{0}^{\infty} \exp \Big\{ -\frac{\lambda^p}{\sigma^p R^p} \Big\} \lambda^{\alpha-1} d\lambda \\
& = 2 \sigma^\alpha R^\alpha \frac{\alpha}{p} \int_{0}^{\infty}e^{-u} u^{\frac{\alpha}{p}-1} du = 2 \sigma^\alpha R^\alpha \Gamma\Big(\frac{\alpha}{p}+1\Big),
\end{aligned}
\end{equation*} 
where $\Gamma$ the gamma function. Hence as we know $\Gamma\Big(\frac{\alpha}{p}+1\Big) \leq \alpha !$ for $p\geq 1$ and $\alpha \geq 2$:
\begin{equation*}
\begin{aligned}
\mathbb{E}\Big[\exp\big\{\lambda W \big\}\Big] & \leq 1 + \sum_{\alpha = 2}^{\infty} \frac{\mathbb{E} |W|^\alpha}{\alpha !} |\lambda|^{\alpha} \leq  1 + 2 \sum_{\alpha = 2}^{\infty} \frac{\Gamma\Big(\frac{\alpha}{p}+1\Big)}{\alpha !} \sigma^\alpha R^\alpha |\lambda|^{\alpha} \\
& \leq  1 + 2 \sum_{\alpha = 2}^{\infty} \sigma^\alpha R^\alpha |\lambda|^{\alpha} \leq  1 + \frac{2 \lambda^2 \sigma^2 R^2 }{1- |\lambda| \sigma R } \\
&   \leq 1 + 4 \lambda^2 \sigma^2 R^2 \leq \exp \Big\{(3 \lambda \sigma R )^2  \Big\}.
\end{aligned}
\end{equation*}

Next, for the case $|\lambda|> 1/[2\sigma R]$, 
we use the Young inequality:
\begin{equation*}
\begin{aligned}
\mathbb{E}\Big[\exp\big\{\lambda W \big\}\Big] & \leq \exp \Big\{ \frac{|\lambda|^q \sigma^q R^q }{q}\Big\} \mathbb{E}\Big[\exp\big\{\frac{|W|^p}{p \sigma^p R^p} \big\}\Big]\\
& \leq \exp \Big\{ \frac{|\lambda|^q \sigma^q R^q }{q} + 2\Big\} \leq \exp \Big\{ \frac{|\lambda|^q \sigma^q R^q }{q} + 2^{q+1} |\lambda|^q \sigma^q R^q \Big\} \\
& \leq \exp \Big\{ (\frac{1}{q} + 2^{q+1}) |\lambda|^q \sigma^q R^q \Big\} \leq \exp \Big\{ 3^q |\lambda|^q \sigma^q R^q \Big\}.
\end{aligned}
\end{equation*}
\end{proof}

As mentioned earlier, most of the approaches in the literature one work with a smooth surrogate of the exponential mgf \cite{Buldygin:2000,Zajkowski:2020}. On the other hand, our approach works directly with the mgf. 

We now state the concentration bounds derived for martingale difference sequences under the mgf bound given in Assumption \ref{assumption:mgf_surrogate}.

\begin{theorem}
Let $(W_t)_t$ be a martingale difference sequence with respect to ~${\cal F}_t=\sigma(W_1,\ldots,W_t)$ (i.e., $\mathbb{E}[W_t| {\cal F}_{t-1}]=0$ for all $t$) such that $W_t$ conditionally on ${\cal F}_{t-1}$ satisfies Assumptions \ref{assump:gradient_noise} and \ref{assumption:mgf_surrogate}. 
For all $T\geq 1$, if we consider $\Sigma_2:=3 \sigma R \sqrt{\sum_{t=1}^T\beta_t^2}$ and $\Sigma_q:=3 \sigma R \big(\sum_{t=1}^T\beta_t^q\big)^{1/q}$, then: 
\begin{equation} \label{eqn:concentration_general}
\begin{aligned}
 \mathbb{P}\Big[ \sum_{t=1}^T \beta_t W_t >\tau \Big] \leq \begin{cases}
 \exp\Big\{-\frac{1}{4}  \big(\frac{\tau}{\Sigma_2}\big)^2 \Big\} \quad & \textit{if} \quad  \tau \leq \frac{\Sigma_2^2}{\sigma R} \\
\exp\Big\{- \frac{\tau}{4\sigma R} \Big\} \quad & \textit{if} \quad   \frac{\Sigma_2^2}{\sigma R} < \tau  \\
\exp\Big\{-\frac{1}{p}  \big(\frac{\tau}{\Sigma_q}\big)^p \Big\} \quad & \textit{if} \quad \frac{\Sigma_q^q}{(\sigma R)^{q-1}} < \tau.
\end{cases}
\end{aligned}
\end{equation}
\end{theorem}

\begin{proof}
Before giving the main idea of the proof, we first notice that by tower property of conditional expectations:
\begin{eqnarray*}
\mathbb{E}\exp\big\{\lambda \sum_{t=1}^T \beta_t W_t \big\} &=& \mathbb{E}\Big\{ \mathbb{E}\Big[\exp\big\{\lambda \beta_T W_T \big\}\Big|{\cal F}_{T-1}\Big]\exp\Big[\lambda\sum_{t=1}^{T-1} \beta_t W_t\Big]\Big\}\\
&\leq&  \mathbb{E}\Big\{\prod_{t=1}^{T} \mathbb{E}\Big[\exp\Big(\lambda \beta_t W_t \Big)\Big| {\cal F}_{t-1}\Big] \Big\}.
\end{eqnarray*}

Hence, we start the proof by using the standard Cr\'amer-Chernoff bound, in conjunction with Lemma \ref{lemma:mgf_surrogate}:
\begin{eqnarray*}
\mathbb{P}\Big[ \sum_{t=1}^T \beta_t W_t >\tau \Big] 
& \leq & \inf_{\lambda \in (0,1/2\sigma R]}\mathbb{P}\Big[ \exp\big\{\lambda \sum_{t=1}^T \beta_t W_t\big\} >\exp(\lambda\tau)\Big] \\
&\leq& \inf_{\lambda \in (0,1/2\sigma R]}    \mathbb{E}\Big\{\prod_{t=1}^{T} \mathbb{E} \left[\exp \Big(\lambda \beta_t W_t \Big) \Big|{\cal F}_{t-1} \right]\Big\}   \exp(-\lambda\tau)\\
&\leq& \inf_{\lambda \in (0,1/2\sigma R]}\exp\Big\{ (3 \lambda \sigma R)^2 \sum_{t=1}^T \beta_t^2  \Big\} \exp(-\lambda\tau)
\end{eqnarray*}
 
If $\tau \leq \frac{\Sigma_2^2}{\sigma R}$, we can minimize the upper bound above, which is attained at $\lambda^{\ast}=\tau/(2\Sigma_2^2)\leq 1/(2\sigma R)$. Therefore:
\[ \mathbb{P}\Big[ \sum_{t=1}^T \beta_t W_t >\tau \Big]
\leq \exp\Big\{-\frac14 \Big(\frac{\tau}{\Sigma_2}\Big)^2 \Big\}. \]
Else when $\tau > \frac{\Sigma_2^2}{\sigma R}$,  we just consider $\lambda = 1/(2\sigma R)$:
\[ \mathbb{P}\Big[ \sum_{t=1}^T \beta_t W_t >\tau \Big] \leq \exp\Big\{\frac{\Sigma_2^2}{4\sigma^2 R^2}- \frac{\tau}{2\sigma R} \Big\} \leq \exp\Big\{- \frac{\tau}{4\sigma R} \Big\} . \]

Similarly for all $\tau > 0$:
\[ \mathbb{P}\Big[ \sum_{t=1}^T \beta_t W_t >\tau \Big] 
\leq \inf_{\lambda \geq 1/2\sigma R}\exp\Big\{ (3 \lambda \sigma R)^q \sum_{t=1}^T \beta_t^q  \Big\} \exp(-\lambda\tau) \]

If $\tau \geq \frac{q \Sigma_q^q}{(2 \sigma R)^{q-1}}$, the infimum above is attained at $\lambda^{\ast}=(\tau/q\Sigma_q^q)^{\frac{1}{q-1}}$, which lies in the interval $[1/[2\sigma R],+\infty)$ and since $q\geq 2$:
\[ \mathbb{P}\Big[ \sum_{t=1}^T \beta_t W_t >\tau \Big] \leq \exp\Big\{-\frac{1}{p} \left(\frac{1}{q} \right)^{\frac{1}{q-1}} \Big(\frac{\tau}{\Sigma_q}\Big)^p \Big\} \leq \exp\Big\{-\frac{1}{p} \Big(\frac{\tau}{\Sigma_q}\Big)^p \Big\}. \]

else if $\tau \leq \frac{q \Sigma_q^q}{(2 \sigma R)^{q-1}}$ we just consider $\lambda = 1/2\sigma R$:\[ \mathbb{P}\Big[ \sum_{t=1}^T \beta_t W_t >\tau \Big] \leq \exp\Big\{\frac{\Sigma_q^q}{2^q\sigma^q R^q}- \frac{\tau}{2\sigma R} \Big\} \leq \exp\Big\{- \frac{\tau}{p\sigma R} \Big\} . \]
\end{proof}

\begin{theorem}\label{thm:hp_concentration_p} For an algorithm working in the composite oracle model, let's consider $W_t:=\langle \Delta(x_t),x^{\star}-x_t\rangle$. Suppose $W_t$ conditionally on ${\cal F}_{t-1}:=\sigma(\xi_1,\ldots,\xi_{t-1})$ (where the stochastic gradient in iteration $t$ is $G(x_t,\xi_{t}$) satisfies assumptions \ref{assump:gradient_noise} and \ref{assumption:mgf_surrogate}. 
Consider a polynomial step-size sequence $\alpha_t := t^n $ with $n\geq 0$. Then for all $T\geq 1 $,

\begin{equation} \label{eqn:concentration_final}
\begin{aligned}
 \mathbb{P}\Big[\frac{1}{A_{T}} \sum_{t=1}^T \alpha_t W_t \gtrsim \frac{\Omega  \sigma R }{\sqrt{T}} \Big] \leq 
 \begin{cases}
 \exp\Big\{-\frac{1}{4}   \Omega^2 \Big\} \quad & \textit{if} \quad  \Omega \lesssim \sqrt{T} \\
\exp\Big\{-\frac{1}{p}    (T^{\frac12-\frac1p}\Omega)^p \Big\} \quad & \textit{if} \quad \Omega \gtrsim \sqrt{T}. 
\end{cases}
\end{aligned}
\end{equation}
with some constants that depend on $(n,q)$ only.

\end{theorem}

\begin{proof}
The idea of the proof is to apply the previous result with $\beta_t = \frac{\alpha_t}{\Lambda}$ and we define $\Lambda$ by the identity
\[ \frac{\| \alpha\|_q^q}{\Lambda^q} = \frac{\| \alpha\|_2^2}{\Lambda^2}. \]
Then we obtain after multiplying $\tau$ by $\sigma R$ :
\begin{equation*} 
\begin{aligned}
 \mathbb{P}\Big[ \sum_{t=1}^T \alpha_t W_t >\tau \Lambda \sigma R \Big] \leq \begin{cases}
 \exp\Big\{-\frac{1}{4}  \frac{\Lambda^2}{\| \alpha \|_2^2} \tau^2 \Big\} \quad & \textit{if} \quad  \tau \leq  \frac{\| \alpha\|_{2}^{2}}{\Lambda^2} \\
\exp\Big\{-\frac{1}{p}  \frac{\Lambda^p}{\| \alpha \|_q^p}\tau^p \Big\} \quad & \textit{if} \quad   \tau>\frac{\| \alpha\|_{q}^{q}}{\Lambda^q}.
\end{cases}
\end{aligned}
\end{equation*}

We finish the proof by considering $\tau = \Omega \sqrt{T}$ and by noticing that:\[A_{T}\propto T^{n+1},\Lambda \propto T^n, \| \alpha \|_2 \propto T^{n+\frac12} ,\| \alpha\|_q \propto T^{n+\frac1q} .\]
\end{proof}

\section{Details of proofs in Section \ref{sec:algorithms}} \label{app:sec_alg}

Since the stochastic terms for NACSMD and ACSMD are almost the same, we will consider the following notation:

\begin{equation*}
\begin{aligned}
\begin{cases}
\mathcal{C}_{T}^{e}  := & A_{T}^{-1} \sum_{t=1}^{T} \frac{\| \Delta(x_t)\|_{\ast}^{p}}{p \mu^{p/q}} 
 \left(\frac{\alpha_{t}^{q}}{\gamma_t}\right)^{p/q}\\
\mathcal{C}_{T}^{p}  := & A_{T}^{-1} \sum_{t=1}^{T} \alpha_t \langle \Delta(x_{t}) , x -x_{t}  \rangle  \\
\end{cases}
\end{aligned}
\end{equation*}

The only difference in the acceleration method is that we have $\Delta(x_{t}^{md})$ instead of $\Delta(x_{t})$, but notice their stochastic noise is of the same kind.

\subsection{Proof of the concentration bound in Theorem  \ref{thm:NACSMD}}
\label{Proof Concentration}

\begin{proof} 

To simplify the notation, in this subsection we will note: $ \beta_t := \frac{\alpha_{t}^{p}}{\gamma_{t}^{p/q}}. $ From Markov inequality, we know that for $\Omega > 0 $:

\begin{equation*}
\begin{aligned}
\mathbb{P}\left(\sum_{t=1}^{T} \beta_t \| \Delta(x_t)\|_{\ast}^{p} \geq (1+\Omega) \sum_{t=1}^{T} \beta_t  \sigma^{p} \right)
= &  \mathbb{P} \left(\exp \left(\frac{\sum_{t=1}^{T} \beta_t  \frac{\| \Delta(x_t)\|_{\ast}^{p}}{\sigma^{p}}}{\sum_{t=1}^{T} \beta_t  }\right) \geq e^{1+\Omega} \right) \\
\leq &  \mathbb{E}\left[\exp \left(\frac{\sum_{t=1}^{T} \beta_t  \frac{\| \Delta(x_t)\|_{\ast}^{p}}{\sigma^{p}}}{\sum_{t=1}^{T} \beta_t  }\right) \right] \frac{1}{e^{1+\Omega}} \\
\end{aligned}
\end{equation*}

Now we use convexity of exponential and linearity of the expectation:
\[ \mathbb{E}\left[\exp \left(\frac{\sum_{t=1}^{T} \beta_t  \frac{\| \Delta(x_t)\|_{\ast}^{p}}{\sigma^{p}}}{\sum_{t=1}^{T} \beta_t  }\right) \right]  
\leq  \mathbb{E}\left[\frac{\sum_{t=1}^{T} \beta_t  \exp(\frac{\| \Delta(x_t)\|_{\ast}^{p}}{\sigma^{p}})}{\sum_{t=1}^{T} \beta_t  }\right]
\leq   \exp(1) \]

We obtain that: 
\[ \mathbb{P}\left( \mathcal{C}_{T}^{e} \geq (1+\Omega)  \frac{\sigma^{p}}{p \mu^{p/q}} A_{T}^{-1} \sum_{t=1}^{T}  \frac{\alpha_{t}^{p}}{\gamma_{t}^{p/q}} \right) \leq \exp(-\Omega) \]

Since under step-sizes schedule \ref{configuration:stepsize_linear} or \ref{configuration:stepsize_acc}, we have $A_{T}^{-1} \sum_{t=1}^{T} \frac{\alpha_{t}^{p}}{\gamma_{t}^{p/q}} \propto \frac{1}{T^{p/q}}$, combining with Theorem \ref{thm:hp_concentration_p}, we know if $\Omega \lesssim \sqrt{T} $:

\[ \mathbb{P}\left( \mathcal{C}_{T}^{e} + \mathcal{C}_{T}^{p} \gtrsim (1+\Omega)  \frac{\sigma^{p}}{(\mu T)^{p/q}} + \frac{\Omega \sigma R}{\sqrt{T}} \right) \leq \exp(-\Omega) + \exp\Big\{-\frac{1}{4}   \Omega^2 \Big\}\]

and if  $\Omega \gtrsim \sqrt{T} $ :
\[ \mathbb{P}\left( \mathcal{C}_{T}^{e} + \mathcal{C}_{T}^{p} \gtrsim (1+\Omega)  \frac{\sigma^{p}}{(\mu T)^{p/q}} + \frac{\Omega \sigma R}{\sqrt{T}} \right) \leq \exp(-\Omega) + 
\exp\Big\{-\frac{1}{p}    (T^{\frac12-\frac1p}\Omega)^p \Big\}.  \]

\end{proof}

\end{document}